\documentclass[lettersize,journal]{IEEEtran}

\IEEEoverridecommandlockouts

\usepackage{hyperref}
\usepackage{url}
\usepackage{graphics} % for pdf, bitmapped graphics files
\usepackage{times}
\usepackage{booktabs}       % professional-quality tables
\usepackage{amsfonts}       % blackboard math symbols
\usepackage{nicefrac}       % compact symbols for 1/2, etc.
\usepackage{microtype}      % microtypography
\usepackage{algorithm,multirow,xcolor}
\usepackage{algorithmicx}
\usepackage{algpseudocode}
\usepackage{mathtools}
\usepackage{graphicx}
\usepackage{cases}
\usepackage{array}
\usepackage{color}
\usepackage{float}
\usepackage{bm}
\usepackage[sort]{cite}
\usepackage{tabularx}
\usepackage{threeparttable}
\usepackage{wrapfig}
\usepackage{amsmath}
\usepackage{amssymb}
\usepackage{amsthm, amssymb}
\usepackage{soul}
\usepackage{subfigure}
\usepackage{epstopdf}
\usepackage{verbatim}
\usepackage{balance}

\newtheorem{theorem}{Theorem}
\newtheorem{remark}{Remark}
\newtheorem{proposition}{Proposition}
\usepackage[skip=5pt]{caption}


\title{\LARGE \bf
Heuristic Predictive Control for Multi-Robot Flocking in \\ Congested Environments
}

\author{Guobin Zhu\textsuperscript{1,2}, Qingrui Zhang\textsuperscript{1}, Bo Zhu\textsuperscript{1,3}, Tianjiang Hu\textsuperscript{1} % <-this % stops a space
\thanks{This work is supported by the National Nature Science Foundation of China under Grant 62103451, 62373386, Guang Dong Basic and Applied Basic Research Foundation  under Grant 2024A1515012408, and Shenzhen Science and Technology Program JCYJ20220530145209021. (Corresponding author: Qingrui Zhang)}
\thanks{$^{1}$School of Aeronautics and Astronautics, Shenzhen Campus of Sun Yat-sen University, Shenzhen 518107, China}
\thanks{$^{2}$School of Automation Science and Electrical Engineering, Beihang University, Beijing 100000, China.}
\thanks{$^{3}$Center for Advanced Control and Smart Operations (CACSO), Nanjing University, Suzhou  215163, China.}
}

\begin{document}
\bstctlcite{IEEEexample:BSTcontrol}
	\maketitle
	\thispagestyle{empty}
	\pagestyle{empty}
	
%%%%%%%%%%%%%%%%%%%%%%%%%%%%%%%%%%%%%%%%%%%%%%%%%%%%%%%%%%%%%%%%%%%%%%%%%%%%%%%%
\begin{abstract}
Multi-robot flocking possesses extraordinary advantages over a single-robot system in diverse domains, but it is challenging to ensure safe and optimal performance in congested environments. Hence, this paper is focused on the investigation of distributed optimal flocking control for multiple robots in crowded environments. A heuristic predictive control solution is proposed based on a Gibbs Random Field (GRF), in which bio-inspired potential functions are used to characterize robot-robot and robot-environment interactions. The optimal solution is obtained by maximizing a posteriori joint distribution of the GRF in a certain future time instant. A gradient-based heuristic solution is developed, which could significantly speed up the computation of the optimal control. Mathematical analysis is also conducted to show the validity of the heuristic solution. Multiple collision risk levels are designed to improve the collision avoidance performance of robots in dynamic environments. The proposed heuristic predictive control is evaluated comprehensively from multiple perspectives based on different metrics in a challenging simulation environment. The competence of the proposed algorithm is validated via the comparison with the non-heuristic predictive control and two existing popular flocking control methods. Real-life experiments are also performed using four quadrotor UAVs to further demonstrate the efficiency of the proposed design.
\end{abstract}


%%%%%%%%%%%%%%%%%%%%%%%%%%%%%%%%%%%%%%%%%%%%%%%%%%%%%%%%%%%%%%%%%%%%%%%%%%%%%%%%
\section{Introduction}
\IEEEPARstart{I}{n} nature, gregarious animals, such as birds, fish, and insects, exhibit impressive collective behaviors to perform complicated tasks \cite{bonabeau1999swarm,8686188,zhang2011spill,zhang2017aerodynamics}. Among different collective behaviors, flocking has intrigued researchers for a long time due to its great potential in diverse applications, such as search and rescue, region monitoring, and package transportation, \emph{etc.} \cite{6732930, Hu2021TRO, zhang2021robust, 9914633}. It is renowned that the flocking behaviors of gregarious animals are distributed, flexible, scalable, and computationally parsimonious. Such properties are fascinating for the coordination of multiple robots with limited on-board sensing, communication, and computation capabilities. Hence, tremendous endeavors have been made to transfer the flocking capabilities of living beings to robots since the seminal work by Reynolds \cite{reynolds1987flocks}. However, it is still challenging to establish efficient distributed flocking behaviors for collective robots merely using local interactions and on-board computation capabilities, especially in congested environments. 
\begin{figure}[tbp]
    \centering
    \includegraphics[width = 1\linewidth]{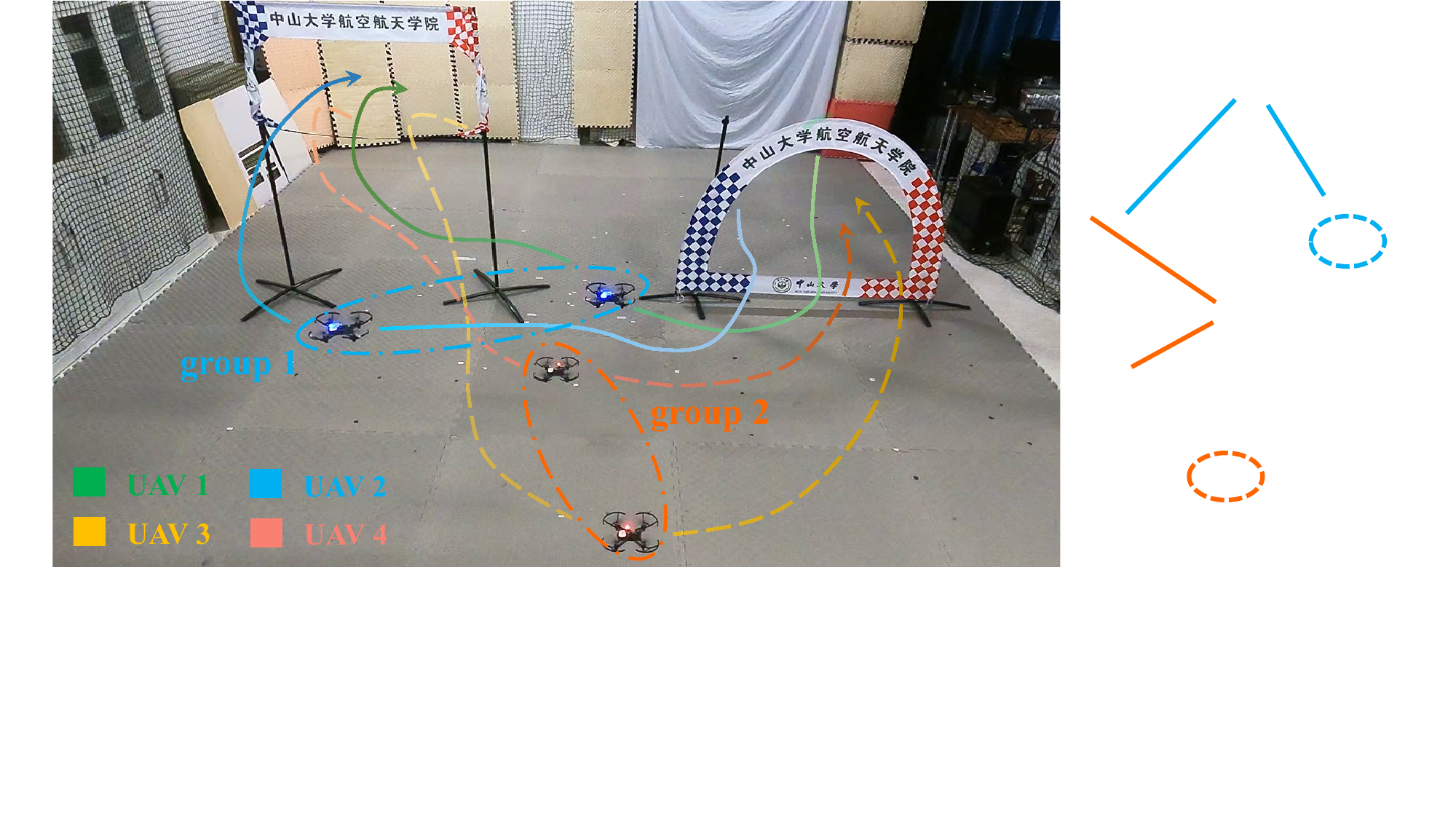}
    \caption{Two UAV groups flying seperately through openings. UAVs 1 and 2 belong to one group, while UAVs 3 and 4 belong to the other group.}
    \label{fig:real_trajectory}
\end{figure}

% In the seminal work by Reynolds \cite{reynolds1987flocks}, three heuristic rules -- namely, \emph{separation}, \emph{alignment}, and \emph{cohesion} -- were proposed to mimic bird flocking. Since then, a significant number of flocking control methods have been proposed, such as artificial potential fields (APFs) \cite{6876179,olfati2006flocking,vasarhelyi2018optimized,zhan2013flocking,7490388,hu2020distributed} and optimization-based methods \cite{6293885,6853439,yuan2017outdoor}.  In APF-based methods, virtual repulsive and attractive forces are introduced among robots based on the gradients of certain potential fields. APF-based flocking control belongs to heuristic methods with distributed fashion and simple formulation, thereby resulting in low computation complexity \cite{bennet2010distributed}. However, APF-based methods would often generate unnatural and oscillatory motions, especially for flocking in cluttered environments \cite{soria2021Nature}. Tedious tuning efforts are indispensable for the generation of reasonable flocking behaviors by APFs, including inter-robot distance maintenance, collision avoidance, and motion alignment \cite{bennet2010distributed}.  APF-based methods cannot explicitly handle any physical constraints. 
In the seminal work by Reynolds \cite{reynolds1987flocks}, three heuristic rules -- namely, \emph{separation}, \emph{alignment}, and \emph{cohesion} -- were proposed to mimic bird flocking. Since then, numerous flocking control methods have been developed, including artificial potential fields (APFs) \cite{6876179,olfati2006flocking,vasarhelyi2018optimized,zhan2013flocking,7490388,hu2020distributed} and optimization-based approaches \cite{6293885,6853439,yuan2017outdoor}. APF methods use virtual repulsive and attractive forces derived from potential field gradients to guide robot movements. These methods are heuristic and distributed, leading to low computational complexity \cite{bennet2010distributed}. However, APF-based methods often result in unnatural and oscillatory movements, particularly in cluttered environments \cite{soria2021Nature}. They require extensive tuning to achieve desirable flocking behaviors, such as maintaining inter-robot distances, avoiding collisions, and aligning motions \cite{bennet2010distributed}. Moreover, APF methods cannot explicitly handle physical constraints.

In comparison to APFs, optimization-based methods have demonstrated their extraordinary promise in improving flocking performance (\emph{e.g.}, the convergence rate and motion smoothness) and handling diverse constraints (\emph{e.g.}, physical limitations and environment restrictions) \cite{zhang2015model,7574368}. One of the most common optimization-based flocking control methods is model predictive control (MPC) which solves an online optimization problem recursively based on predicted trajectories by a mathematical model. The recursive online optimization of MPC is resolved in a centralized or distributed fashion \cite{9662427,zhang2015model,8429104}. The centralized MPC requires a data-processing center that calculates the control inputs for each robot based on global information, so it lacks reliability, flexibility, and robustness \cite{6293885}. Unlike centralized MPC, distributed MPC calculates the control inputs for each robot separately using an on-board processor and local state information \cite{8877998,shi2021advanced}. Hence, distributed MPC is scalable for a large-scale flock and robust against unexpected communication link failures \cite{8877998,shi2021advanced}. 
% However, the existing distributed MPC is computationally expensive, as the control decisions are calculated by minimizing an accumulated cost function over a given prediction horizon.  It also requires robots to negotiate their planned full control decisions with neighbors in a sequential manner, which demands very high communication bandwidth for real-time applications. 
However, existing distributed MPC is computationally expensive due to the need to minimize an accumulated cost function over a prediction horizon. Additionally, it requires sequential negotiation of control decisions among robots, demanding high communication bandwidth for real-time applications.

% In this paper, an efficient distributed predictive control algorithm is proposed for multi-robot flocking in dynamic challenging environments. The robot swarm is characterized by a dynamic Gibbs Random Field (GRF) that is a collection of random variables at sites with their joint probability proportional to potential energy \cite{koller2009probabilistic}. Such spatially correlated probabilities in GRF can be used to describe the behavioral rules of robots in flocking, where robots are assumed to move at a discrete time with coordinates considered as (mobile) sites of the random field \cite{xi2006gibbs}. 
This paper proposes an efficient distributed predictive control algorithm for multi-robot flocking in dynamic, challenging environments. The robot swarm is characterized by a Gibbs Random Field (GRF), where the joint probability of random variables at different sites is proportional to potential energy \cite{koller2009probabilistic}. Such spatially correlated probabilities are used to describe the behavioral rules of robots in flocking, with robots moving at discrete times and their coordinates serving as mobile sites of the random field \cite{xi2006gibbs}.
The robot-robot and robot-environment interactions are characterized using potential energies defined by APFs in GRF \cite{tan2010decentralized}. The objective of flocking control in GRF turns into the inference of the best control via the minimization of the potential energies inside a neighborhood \cite{rezeck2021cooperative,fernando2021online}. The minimization problem is resolved online in a recursive fashion similar to MPC. 
% Hence, GRF provides a compelling framework that leverages the advantages of both APFs and optimization-based methods. Both the predictive capabilities and the optimization process contribute to the performance improvement of the GRF-based flocking control over APFs. 
Thus, GRF provides a compelling framework that leverages the benefits of APFs and optimization-based methods.
% , with its predictive capabilities and optimization process enhancing performance over traditional APFs

It should be noted that the converged flocking by the best control also represents the maximum a posteriori (MAP) distribution of the GRF. 
% Similar to the method in \cite{fernando2021online}, the posterior distribution of the GRF is approximated in a distributed way using predicted positions of neighboring robots. The optimization of the GRF-based predictive control in this paper is, therefore, based on the configuration of robots at a future time instant instead of using trajectories in a horizon. Control space is discretized for fast online trajectory optimization, which is similar to the dynamic window approach \cite{580977}. Hence, it is more computationally efficient than conventional MPC. However, both our simulation and experiments have illustrated that the original GRF-based flocking control as in \cite{xi2006gibbs,tan2010decentralized,rezeck2021cooperative,fernando2021online} is not efficient enough for a large-scale flock from both computation cost and collision avoidance perspectives. To ensure good performance, the control discretization must be fine enough, which will in turn lead to a dramatic increase in the computation burden of the optimization process in the GRF-based control inference. This dilemma is resolved in this paper by introducing the APF-based inputs as the heuristic initial guess with the consideration of the fact that the major potential energies of the GRF-based control are defined based on APFs.
Similar to the method in \cite{fernando2021online}, the posterior distribution is approximated in a distributed manner using predicted positions of neighboring robots. Thus, the GRF-based predictive control optimization in this paper relies on the configuration of robots at a future time instant rather than using trajectories over a horizon. The control space is discretized for fast online trajectory optimization, similar to the dynamic window approach \cite{580977}, making it more computationally efficient than conventional MPC. However, our simulations and experiments show that the original GRF-based flocking control as in \cite{xi2006gibbs,tan2010decentralized,rezeck2021cooperative,fernando2021online} is not efficient enough for a large-scale flock from both computation cost and collision avoidance perspectives. High performance requires fine control discretization in the GRF-based optimization, thus leading to the increases of computational burden. This dilemma is resolved in this paper by introducing the APF-based inputs as the heuristic initial guess, given that the major potential energies of the GRF-based control are defined by APFs. Theoretical analysis is presented to illustrate the validity of such a choice. The collision avoidance performance is improved by defining multi-level risky areas with the consideration of the relative motion between two non-cooperative robots. To demonstrate the efficiency of the proposed design, we design an evaluation scenario, where two independent flocks of robots go through multiple openings as shown in Fig. \ref{fig:real_trajectory}. Such a scenario is adversary, competitive, and congested for each group of robots. Both simulation and real-life experiments are performed to thoroughly validate the whole design.\footnote{\href{https://drive.google.com/file/d/1E-bzEHQ-o1WAXZUTxUgXYtRg30VAxhCR/view}{\emph{Link of simulation and experiment videos.}}} In summary, the overall contributions of this paper are threefold:
\begin{itemize}
  \item [1)]
  A heuristic predictive flocking control algorithm is developed for multiple robots in dynamic challenging environments. It employs the gradient of APFs as a heuristic initial guess for the computation of the optimal control. Such a heuristic design would dramatically reduce the computation cost.
  \item [2)]
  Multiple collision risk levels are designed based on the relative motion between the ego robot and a non-cooperative robot (\emph{e.g.}, either an intelligent robot in a different group or a moving/static obstacle). With the subtle discrimination of different risky levels, collision avoidance performance can be improved significantly.
  \item [3)]
  A rigorous theoretical analysis is provided to illustrate the validity of using the gradient of APFs as the initial input guess for the optimization process of the GRF-based control. We also mathematically illustrate the convergence of the distributed approximation of the maximum a posteriori (MAP) distribution of the GRF.
\end{itemize}

The rest of the article is organized as follows. The system model and Gibbs random fields are described in Section \ref{sec:Preliminaries}. The whole efficient control design is then proposed in Section \ref{sec:Methodology}. Following that, simulations and real-world experiments are presented in Section \ref{sec:Experiments}. Finally, a summary of this work is given in Section \ref{sec:Conclusion}.

\section{Preliminaries \label{sec:Preliminaries}} 
\subsection{System Modeling} \label{subsec:System Modeling}
A set of $N$ homogeneous robots are considered in this paper. The robot set is specified as $\mathcal{A}=\{1,\ldots, N\}$. Let the state vector of the robot $i \in \mathcal{A}$ be $\mathbf{x}_i = [\mathbf{p}_i^{T}, \mathbf{v}_i^{T}]^{T} \in \mathbb{R}^{6}$, where $\mathbf{p}_i\in \mathbb{R}^{3}$ and $\mathbf{v}_i\in \mathbb{R}^{3}$ are the position and velocity vectors in the inertial frame, respectively. In most flocking models, a robot is assumed to have second-order dynamics with its motion updated by changing the acceleration. Hence, the following discrete-time robot model is employed.
% for the flocking control design. 
\begin{equation}\label{eq:state-space model}
    \mathbf{x}_i(t_k + \triangle t) = \mathbf{A}\mathbf{x}_i(t_k) + \mathbf{B}\mathbf{u}_i(t_k) 
\end{equation}
where $\mathbf{u}_i(t_k)$ is the control input vector, $\triangle t$ is the step size, $\mathbf{A} = \left[\begin{array}{cc}1, \;& \triangle t \\ 0, \;& 1 \end{array} \right] \otimes \mathbf{I}_{3}, \mathbf{B} = [\triangle t^2/2, \; \triangle t ]^T \otimes \mathbf{I}_{3}$, $\otimes$ is the Kronecker product, and $t_k$ denotes the $k$-th time instant. Both the velocity and acceleration of robot $i$ are bounded with $\Vert \mathbf{v}_i \Vert \leq v_{max}$, $\Vert \mathbf{u}_i \Vert \leq u_{max}$, and $v_{max}$, $u_{max}>0$. 

The visibility or interaction sphere of robot $i$ is defined as a region surrounded by a ball with a constant radius $r_s$, also called the perception radius for a robot in flocking. The neighborhood of robot $i$ is chosen to be the set of neighboring robots within its interaction sphere. Mathematically, the neighborhood of robot $i$ is denoted as $\mathcal{N}_i = \{j \in \mathcal{A} \vert \Vert \mathbf{p}_{ji} \Vert \leq r_s\text{, } j \neq i\}$, where $\mathbf{p}_{ji} = \mathbf{p}_j - \mathbf{p}_i$ is the relative position of robot $j$ to $i$. It is assumed that the interaction between any two robots is bi-directional, so there exists $j \in\mathcal{N}_i \Leftrightarrow i \in \mathcal{N}_j$. 
% ($j \notin \mathcal{N}_i \Leftrightarrow i \notin \mathcal{N}_j$) 
The interactions among robots in flocking are, therefore, characterized by an undirected graph $\mathcal{G}(\mathcal{A},\mathcal{E})$, where $\mathcal{A}$ is the robot set and $\mathcal{E} = \{(i,j) \in \mathcal{A} \times \mathcal{A}\vert j \in \mathcal{N}_i\text{, } \forall i\in \mathcal{A} \}$ is the set of interaction edges. The neighbor set $\mathcal{N}_i$ depends on the relative positions among robots, so $\mathcal{G}$ is dynamically changing.

The virtual robot concept in \cite{olfati2006flocking}, denoted as a $\beta$-robot, is borrowed to characterize static obstacles in the environment. The $\beta$-robot denotes the closest point to the robot $i$ on the surface of an obstacle, so it has no size. The set of static obstacles perceived by robot $i$ are specified as $\mathcal{O}_{s,i} = \{\beta\in\mathcal{O}_{s} \vert \Vert \mathbf{p}_{\beta}-\mathbf{p}_{i} \Vert \leq r_s \}$, where $\mathcal{O}_{s}$ is the set of all static obstacles and $\mathbf{p}_{\beta}$ is the position of the $\beta$-robot \cite{olfati2006flocking}. With the $\beta$-robot concept, both regular and irregular static obstacles can be easily addressed in the same framework. Similar concepts can also be found in \cite{cole2018reactive}. The dynamic obstacles, which are mostly non-cooperative robots, are also called $\beta$-robot with a little abuse of notations. However, all dynamic obstacles are assumed to be enclosed by a ball with a radius $r_\beta$.
The dynamic obstacle set observed by robot $i$ is denoted by $\mathcal{O}_{d,i} = \{\beta \in\mathcal{O}_{d}\vert \Vert \mathbf{p}_{\beta}-\mathbf{p}_{i} \Vert \leq r_s + r_\beta \}$ where $\mathcal{O}_{d}$ is the set of all dynamic obstacles and $\mathbf{p}_{\beta}$ is the position of the dynamic obstacle $d$. Hence, the set of all perceived obstacles by robot $i$ is given by $\mathcal{O}_{i}=\mathcal{O}_{s,i}	\cup \mathcal{O}_{d,i}$.

To be more realistic, we assume that robots have sizes, thus the safety conditions need to be considered in the robot swarm. Make a conservative assumption that the robot can be represented by a ball of radius $r_c$, the safety constraint for robot $i$ may then be formulated as $\Vert \mathbf{p}_{ij} \Vert > 2r_c$, $j \in \mathcal{N}_i$; $\Vert \mathbf{p}_{i\beta} \Vert > r_c$, $\beta \in \mathcal{O}_{s,i}$; $\Vert \mathbf{p}_{i\beta} \Vert > r_c + r_\beta$, $\beta \in \mathcal{O}_{d,i}$. 

\subsection{Gibbs Random Field} \label{subsec:Gibbs Random Field}
Gibbs random field (GRF) is a set of random variables with a spatial Markov property and strictly positive joint probability density. The spatial Markov property allows one to ignore more distant information as soon as immediate local information is provided \cite{rezeck2021flocking}. In GRF, the joint probability of random variables is characterized as functions over cliques in an undirected graph. Consider the undirected graph  $\mathcal{G}(\mathcal{A},\mathcal{E})$ as given in Subsection \ref{subsec:System Modeling}. A node subset $\mathcal{Q} \subseteq \mathcal{A}$ is called a clique if all elements in $\mathcal{Q}$ are neighbors to one another. A GRF on $\mathcal{G}$ is, therefore, a collection of random variables $X = \{X_i\}_{i\in\mathcal{A}}$ indexed by $\mathcal{A}$, where $X_i$ is the spatial site of the $i$-th random variable in $\mathcal{A}$. Hence, an instance of $X$ is a spatial configuration of the GRF, which should obey the following Gibbs distribution. 
\begin{equation}
    p(X) = Z^{-1}\prod _{\mathcal{Q} \in \mathcal{C}} \varPsi _{\mathcal{Q}} (X_{\mathcal{Q}}) \label{eq:GibbsDistr}
\end{equation}
where $\mathcal{C}$ is the set of cliques on the graph $\mathcal{G}$, $X_{\mathcal{Q}}$ denotes random variables over $\mathcal{Q}$, $\varPsi _{\mathcal{Q}}$ is a non-negative clique potential, and $Z = \sum _{X} \prod _{\mathcal{Q} \in \mathcal{C}} \varPsi _{\mathcal{Q}} (X_{\mathcal{Q}})$ is a normalization factor. 

The clique potential $\varPsi _{\mathcal{Q}}$ has an exponential form of $\varPsi _{\mathcal{Q}}=\exp (-\psi _{\mathcal{Q}}(X_{\mathcal{Q}}))$ with $\psi _{\mathcal{Q}}(X_{\mathcal{Q}})$ representing the potential energy of the spatial configuration of $X_{\mathcal{Q}}$. In flocking, such a configuration of $X_{\mathcal{Q}}$ corresponds to the position configuration of neighboring robots. Minimum potential energy corresponds to the optimal or stable flocking behaviors of robots. Hence, GRF provides a powerful framework to characterize interactions among neighboring robots. As can be seen, the joint probability is inversely proportional to the potential energy function $\psi _{\mathcal{Q}}(X_{\mathcal{Q}})$. Hence, minimizing potential energy functions $\psi _{\mathcal{Q}}(X_{\mathcal{Q}})$ is equivalent to obtaining the maximum a posteriori possibility (MAP) of the random variables $X$ on a GRF. The predictive control in this paper is designed to minimize the potential energy of the spatial configurations of robot flocks at a future time instant. The potential energy functions $\psi _{\mathcal{Q}}(X_{\mathcal{Q}})$ is defined by the combination of interaction potentials for flocking, collision avoidance potentials, and control performance potentials, \emph{etc}.    

\section{Methodology}\label{sec:Methodology}
In this section, an efficient distributed predictive control strategy is developed for a group of robots in congested environments. A GRF is employed to characterize the spatial correlation or local interactions among UAVs. The optimal control is inferred by maximizing the posterior distribution of the GRF, which corresponds to the minimization of the total potential energy of the multi-robot configuration. Hence, different potential energy components are firstly designed in this section, including the robot interaction energy, collision avoidance energy, and control performance energy, \emph{etc.} The GRF is thereafter constructed based on the designed potential energies according to Section \ref{subsec:Gibbs Random Field}. A distributed iterative algorithm is developed to approximate the MAP of the GRF.
% , which will be used for the determination of the control input.
Eventually, the optimal input for each robot is efficiently calculated via a biased discretization of the control space.

\subsection{Potential Energies}\label{subsec:Potential Energies}
% In this paper, a group of robots are expected to move safely along a reference trajectory as a flock in challenging environments. The safe motion of robots in flocking is generated by the combination of various collective behaviors of multiple robots in flocking, such as flock cohesiveness, inter-robot separation, collision avoidance, motion smoothness, and goal tracking, \emph{etc.} A set of potential energies are, therefore, introduced to characterize the aforementioned flocking behaviors of multiple robots. Those potential energies can also be regarded as cost functions to evaluate the flocking control performance.  More details on the potential energies are given below.
In this paper, a group of robots is expected to move safely along a reference trajectory as a flock in challenging environments. The safe motion is generated by combining various collective behaviors such as flock cohesiveness, inter-robot separation, collision avoidance, motion smoothness, and goal tracking. A set of potential energies are, therefore, introduced to characterize the aforementioned flocking behaviors. More details on these potential energies are provided below.

1) \emph{Inter-robot potential energy}: This energy is used to achieve the simplest collective behavior in an opening environment, such as the cohesion of the swarm on large scales and separation at a close range, which is also called the robot-robot interaction potential energy. The inter-robot potential energy between any robot $i$ and $j$ is, therefore, defined by 
\begin{equation}\label{eq:Attraction and repulsion}
    \psi _{ar} (\mathbf{x}_i, \mathbf{x}_j) = 
    \begin{cases}
        k_a\left(1 + \cos{\left (\frac{\pi\Vert\mathbf{p}_{ij}\Vert}{r_f} \right )} \right),& \text{if} \ \Vert\mathbf{p}_{ij}\Vert \leq k_tr_f \\ \\
        \begin{aligned}
         & -k_a\frac{\pi}{r_f}\sin{(\pi k_t)}\\ 
         &  \times\left(\Vert\mathbf{p}_{ij}\Vert- k_tr_f\right) \\ 
         % + k_a\left( 1 + \cos{k_t\pi} \right)
         &  +k_a\left( 1 + \cos{(k_t\pi) } \right)
        \end{aligned}
        ,& \text{otherwise.}
    \end{cases}
\end{equation}
where $k_a, r_f > 0$, $0 < k_t \leq 2$ and $i,j \in \mathcal{A}$, $i \neq j$. At the characteristic distance $r_f$, the minimum inter-robot potential energy can be obtained as shown in Fig. \ref{fig:sketch_map}(a). By properly increasing the pairwise gain $k_a$, the inclination of the wall in Fig. \ref{fig:sketch_map}(a) can be enhanced, which promotes the maintenance of $r_f$. Note the gradient of the inter-robot potential energy function \eqref{eq:Attraction and repulsion} will change smoothly with respect to the relative distance between two robots. Such smooth changes are more beneficial to flocking control than the potential energy in \cite{7434002}. 
\begin{figure}[tbp]
    \centering
    \includegraphics[width = 1\linewidth]{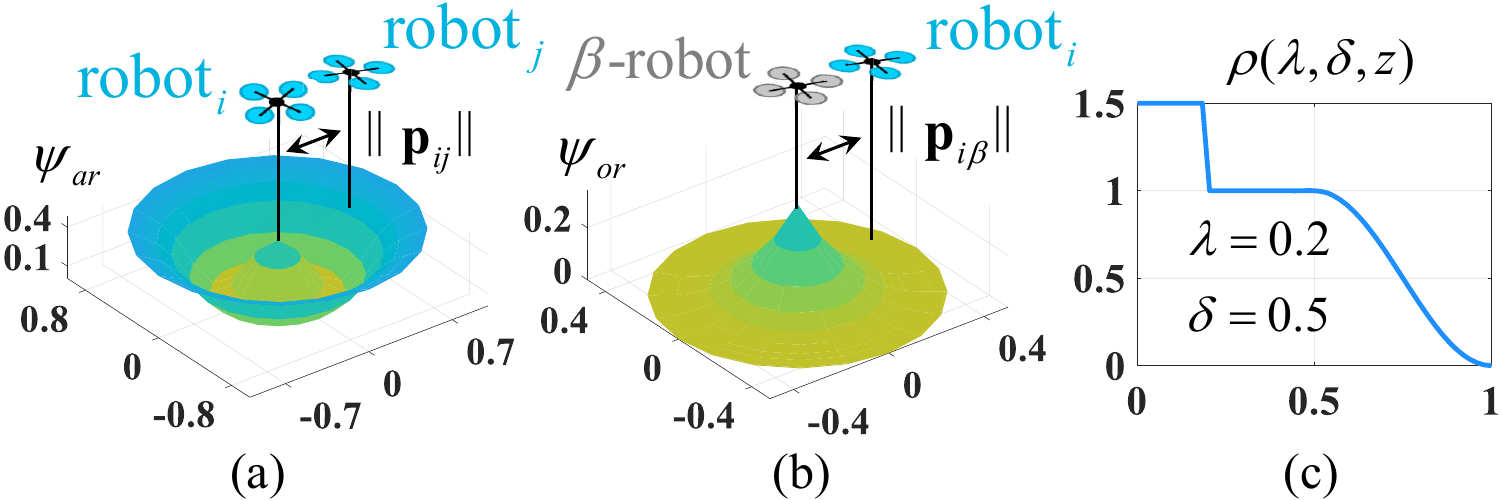}
    \caption{Inter-robot potential energy (a), repulsion energy (b), and transition function (c). In (a), $r_f = 0.421, k_a = 0.2, k_t = 1.95$; in (b), $r_f = 0.421, k_{or} = 0.2$; in (c), $k_{\rho} = 1.5$.}
    \label{fig:sketch_map}
\end{figure}

2) \emph{Obstacle avoidance potential energy}: This energy is employed to enable robots to pass obstacle-rich environments securely and efficiently as illustrated in Fig. \ref{fig:Obstacle_avoidance}. The obstacle avoidance action is designed based on the relative position and velocity of robot $i$ to an obstacle or the $\beta$-robot, namely $\mathbf{p}_{i\beta}$, $\mathbf{v}_{i\beta}$. Let $\theta _{i\beta} = \arccos\left(-\mathbf{p}_{i\beta} \cdot \mathbf{v}_{i\beta}/\left(\Vert \mathbf{p}_{i\beta}\Vert \Vert \mathbf{v}_{i\beta} \Vert\right)\right)$. As shown in Fig. \ref{fig:Obstacle_avoidance}, there is almost no collision if $\theta _{i\beta} \geq \theta _{\mathrm{III}}$, so no avoidance action is required.  If $\theta _{i\beta} < \theta _{\mathrm{III}},r_f < \Vert \mathbf{p}_{i\beta} \Vert < r_s$, a collision is possible, so the robot $i$ has to execute a maneuver to avoid the obstacle. At this stage, the robot is required to move along $\mathbf{v}_{ob}$ calculated by equation (\ref{eq:heuristic obstacle avoidance}). This behavior is described by a direction potential $\psi _{od}$ (\ref{eq:obstacle avoidance}). If $\theta _{i\beta} < \theta _{\mathrm{III}},\Vert \mathbf{p}_{i\beta} \Vert \leq r_f$, the robot $i$ would be subject to an extra repulsive force by $\psi _{or}$.
\begin{equation}\label{eq:Obstacle repulsion}
    \psi _{or}(\mathbf{x}_i,\mathbf{x}_{\beta}) = k_{or}\left( \exp{\left(1 - \sin{\left (\frac{\pi \Vert\mathbf{p}_{i\beta}\Vert}{2r_f}\right )} \right)} - 1\right)
\end{equation}
where $k_{or} > 0$, $i \in \mathcal{A}$ and $\mathbf{x}_{\beta} = [\mathbf{p}_{\beta}^T, \mathbf{v}_{\beta}^T]^T$, $\beta \in \mathcal{O}_i$. The repulsion energy is illustrated in Fig. \ref{fig:sketch_map}(b) and the minimum energy can be achieved when $\Vert \mathbf{p}_{i\beta} \Vert > r_f$. The repulsion gain $k_{or}$ can be expressed as the value corresponding to the green vertex in Fig. \ref{fig:sketch_map}(b). The desired motion direction of a robot is determined by the resultant direction of $\mathbf{v}_{ob}$ and $\nabla _{\mathbf{p}_i} \psi _{or}$, which is denoted as $\mathbf{u}_{go}$. Note that $\nabla _{\mathbf{p}_i} \psi _{or} = 0$ when $\Vert \mathbf{p}_{i\beta} \Vert > r_f$. Therefore, the direction potential $\psi _{od}$ can be designed to denote the difference between the robot's motion direction and its reference $\mathbf{u}_{go}$.
\begin{equation}\label{eq:obstacle avoidance}
    \begin{aligned}
        \psi _{od}(\mathbf{x}_i,\mathbf{x}_{\beta}) =& k_{od} \rho(\lambda,\delta,\frac{r_{\rho}}{(1 + k_{\delta})(r_\beta + r_c)}) \\
        &\times (\exp{(\angle(\mathbf{v}_i,\mathbf{u}_{go}))} - 1)
    \end{aligned}
\end{equation}
where $k_{od}, k_{\delta} > 0$ are the direction gain, risk sector boundary coefficient respectively, $\delta \in (0, 1)$, $r_{\rho} = \Vert \mathbf{p}_{i\beta} \Vert \sin{\theta _{i\beta}}$, $\angle (\mathbf{v}_i,\mathbf{u}_{go}) = \arccos{(\mathbf{v}_i \cdot \mathbf{u}_{go}/}{\Vert \mathbf{v}_i \Vert \cdot \Vert \mathbf{u}_{go} \Vert)}$, $\lambda = 1/(1 + k_{\delta})$. Note $r_\beta=0$ for a static obstacle and $r_\beta>0$ for a dynamic obstacle. The transition function with $k_{\rho} > 1$ is defined as
\begin{equation}\label{eq:bump function}
    \rho (\lambda,\delta,z) =
    \begin{cases}
        k_{\rho}, & \text{if} \ z \in [0,\lambda ) \\
        1, & \text{if} \ z \in [\lambda,\delta) \\
        \frac{1}{2}[1+\cos(\pi\frac{z-\delta}{1-\delta})], & \text{if} \ z\in[\delta,1) \\
        0, & \text{otherwise.}
    \end{cases}
\end{equation}
Consequently, the total obstacle avoidance potential is specified as $\psi _{o} = \psi _{or} + \psi _{od}$. It should be noted that both $\psi _{or}$ and $\psi _{od}$ are essential for ensuring safety in robot-obstacle interactions. The parameter $\psi _{or}$ directly constrains the distance between the robot and the obstacle, effectively maintaining safe separation, although it may lead to bouncing issues \cite{olfati2006flocking}. The introduction of $\psi_{od}$ addresses this concern by providing adequate safety and avoiding rebound problems. Moreover, the adjustable range of the risk sector, determined by $k_{\delta}$, effectively prevents over-rotation \cite{10026865}. Furthermore, our avoid-obstacle strategy only requires knowledge of $\mathbf{p}_{i\beta}$ and $\mathbf{v}_{i\beta}$, 
% which can be easily obtained in the real world using sensors such as cameras, LIDAR, IMU, and others. This accessibility of information ensures the practicality and feasibility of our strategy in real-world implementations.
which can be measured by sensors like cameras, LiDAR, and IMUs. This ensures our strategy's practicality and feasibility in real-world applications.
\begin{figure}[tbp]
    \centering
    \includegraphics[width = 0.9\linewidth]{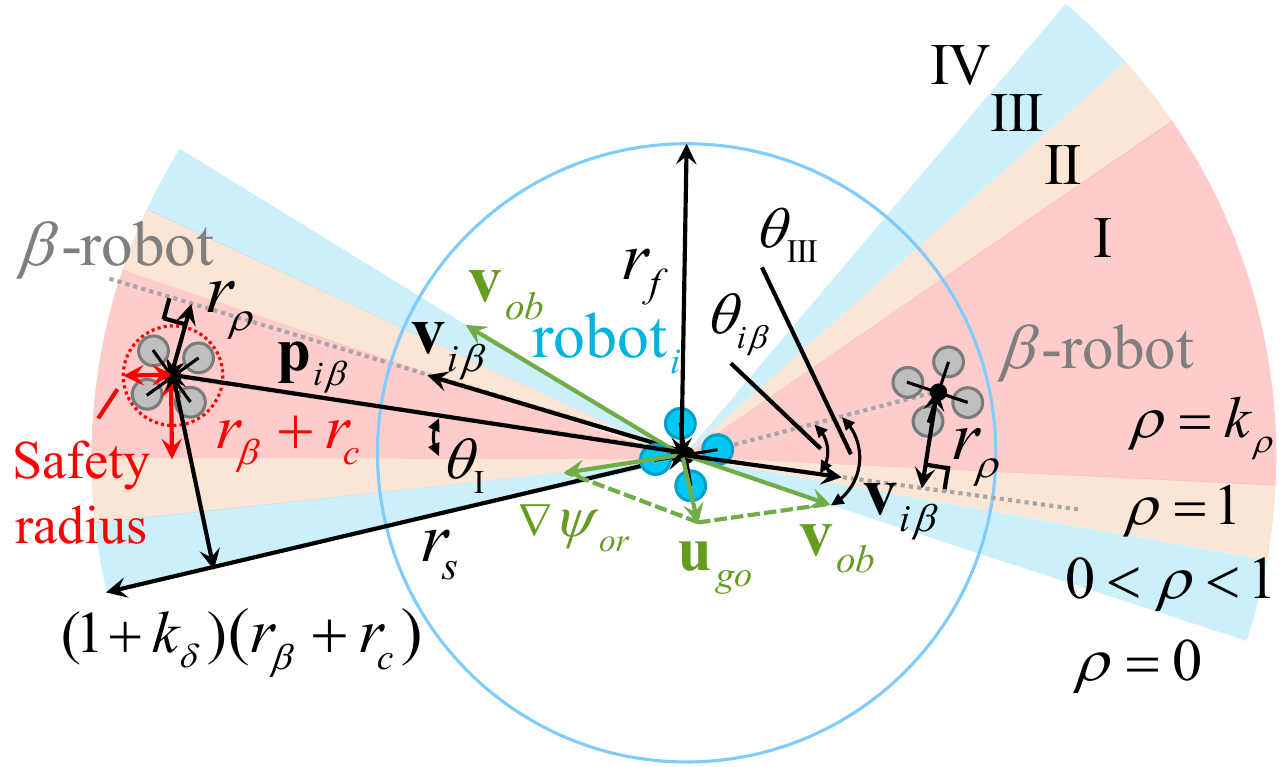}
    \caption{The collision avoidance strategy for robot $i$. The risk level varies across different sectors around the robot, which is represented using sectors with different colors. Sector I represents the highest-risk region, bounded by an angle $\theta_{\mathrm{I}} = \arcsin \left({r_{\beta} + r_c}/{\|\mathbf{p}_{i\beta}\|}\right)$.  Sectors II and III have lower-level risks, which are determined by the function $\rho(\lambda, \delta, z)$. Sector IV is the safe region where $\rho (\lambda,\delta,z) = 0$. The safe boundary is given by $\theta_{\mathrm{III}} = \arcsin \left({(1 + k_{\delta})(r_{\beta} + r_c)}/{\|\mathbf{p}_{i\beta}\|}\right)$. }
    \label{fig:Obstacle_avoidance}
\end{figure}

3) \emph{Goal potential energy}: This potential allows a robot to track a predefined reference trajectory given by $\mathbf{x}_r = [\mathbf{p}_r^T, \mathbf{v}_r^T]^T$. Without loss generality, it is assumed that $\dot{\mathbf{v}}_r = \mathbf{0}$. The purpose of the goal potential is to ensure successful tracking of $\mathbf{x}_r$, which is quantified by using a position potential $\psi _{rp}$ and a velocity potential $\psi _{rv}$ as defined below. 
\begin{equation}\label{eq:Attraction of virtual leader}
    \psi _{rp}(\mathbf{x}_i,\mathbf{x}_{r}) = k_{rp}\left(\exp{(\Vert \mathbf{p}_{ir} \Vert)} - 1\right)
\end{equation}
\begin{equation}\label{eq:Velocity cost}
    \psi _{rv}(\mathbf{x}_i,\mathbf{x}_{r}) = k_{rv}\left(\exp{(\Vert \mathbf{v}_{ir}\Vert)} - 1\right)
\end{equation}
% where $k_{rp}$, $k_{rv} > 0$ are position tracking gain, velocity tracking gain respectively. The position potential $\psi _{rp}$ can attract robots toward time-varying reference points, while velocity potential $\psi _{rv}$ can ensure that the velocity correlation of the swarm during movement is not too low, improving the flocking performance. Hence, the goal potential can be obtained by $\psi _r = \psi _{rp} + \psi _{rv}$.
where $k_{rp}$ and $k_{rv}$ are the position and velocity tracking gains, respectively. The position potential $\psi_{rp}$ attracts robots to time-varying reference points, while the velocity potential $\psi_{rv}$ ensures adequate velocity correlation within the swarm, enhancing flocking performance. Thus, the goal potential is $\psi_r = \psi_{rp} + \psi_{rv}$.

\subsection{GRF and Distributed Approximation}\label{subsec:Iterative Optimization}
With the aforementioned potential energies, a GRF is constructed to model the interactions and collective behaviors of robots in flocking. The random variable set of GRF is denoted by $X = \{X_i\} _{i \in \mathcal{A}}$, where $X_i$ is a random variable associated with the state of robot $i$. It means that $X = \mathbf{x}$ represents a particular field configuration with $\mathbf{x}=\left\{\mathbf{x}_i\right\}_{i \in \mathcal{A}}$ and $\mathbf{x}_i$ as the state of robot $i$. The GRF of interest is constructed based on the interactive relationship among robots and the potential energies introduced in Section \ref{subsec:Potential Energies}. The overall interactive relationship of robots in a flock is characterized by an undirected graph that could be factorized into multiple cliques as discussed in Section \ref{subsec:Gibbs Random Field}. According to  (\ref{eq:GibbsDistr}) and results in Section \ref{subsec:Potential Energies}, the potential for each clique is composed by 
\begin{equation}
    \varPsi _{ar}(X_i = \mathbf{x}_i,X_j = \mathbf{x}_j) = \exp{(-\psi _{ar}(\mathbf{x}_i,\mathbf{x}_j))}
\end{equation}
\begin{equation}
    \varPsi _{o}(X_i = \mathbf{x}_i) = \exp{(-\psi _{o}(\mathbf{x}_i,\mathbf{x}_{\beta}))}
\end{equation}
\begin{equation}
    \varPsi _{r}(X_i = \mathbf{x}_i) = \exp{(-\psi _{r}(\mathbf{x}_i,\mathbf{x}_{r}))}
\end{equation}
Note that $\mathbf{x}_r$ is the target state for the flock, and $\mathbf{x}_\beta$ is the state of the $\beta$-robot. Let 
$p(X=\mathbf{x})$ be the probability of a particular field configuration 
$\mathbf{x}$ in $X$, so we have
\begin{equation}
\begin{aligned}
    p(X=\mathbf{x}) &= \frac{1}{Z}\exp \Bigg (  -\sum _i\sum _{j \neq i} \psi _{ar}(\mathbf{x}_i,\mathbf{x}_j) \\
    & - \sum _i \sum _{\beta \in \mathcal{O}_i} \psi _{o}(\mathbf{x}_i,\mathbf{x}_\beta) -\sum _i \psi _{r}(\mathbf{x}_i,\mathbf{x}_{r}) \Bigg )
\end{aligned} \label{eq:GRF_distribution}
\end{equation}

% With \eqref{eq:GRF_distribution}, the objective of the flocking control is transformed into the determination of the optimal control for each robot to maximize a posteriori distribution of the GRF in a future time instant.  However, the GRF distribution \eqref{eq:GRF_distribution} is a global message that is unavailable for each robot. Hence, it is intractable to distributively infer the optimal control for each robot based on \eqref{eq:GRF_distribution}. To address this issue, we introduce the mean-field theory to generate a distributed approximation to \eqref{eq:GRF_distribution}. The mean field approximation refers to a
% class of variational approximation methods that approximate the true distribution $p(X=\mathbf{x})$ on a graph $\mathcal{G}$ with a simple and tractable distribution $q(X=\mathbf{x})$.  With the mean-field theory, the exact variational inference on $p(X=\mathbf{x})$ is transformed into the inference on $q(X=\mathbf{x})$. The approximated $q(X=\mathbf{x})$ is assumed to be fully factorized, so it can be computed in a distributed fashion.
With \eqref{eq:GRF_distribution}, the goal of flocking control becomes finding the optimal control for each robot to maximize the posterior distribution of the GRF at a future time. However, \eqref{eq:GRF_distribution} represents a global message unavailable to each robot, making it impractical to infer the optimal control distributively. To overcome this, we use mean-field theory to approximate \eqref{eq:GRF_distribution} in a distributed manner. Mean-field theory, a type of variational approximation, approximates the true distribution $ p(X=\mathbf{x}) $ on a graph $ \mathcal{G} $ with a simpler, tractable distribution $ q(X=\mathbf{x}) $. This transforms exact inference on $ p(X=\mathbf{x}) $ to inference on $ q(X=\mathbf{x}) $, which is assumed to be fully factorized. So it can be computed in a distributed fashion.
To ensure feasibility and accuracy, it is required to minimize the KL divergence between $q(X=\mathbf{x})$ and $p(X=\mathbf{x})$, namely $\mathbf{D}(q\vert p) = \mathbb{E} _{q(X)} \left[\ln \left(q(X)/p(X)\right)\right]$ \cite{koller2009probabilistic}. The following theorem exists for the mean field approximation.
\begin{theorem}\label{theo:mean_field}
    Suppose the approximate distribution $q(X)$ can be fully factorized with $q(X) = \prod _{i} q_i(X_i)$. The minimization of the KL divergence $\mathbf{D}(q\vert p)$ can be resolved in an iterative way. The iteration process is guaranteed to converge to a minimum under any initial condition.
\end{theorem}
\begin{proof}
    See Appendix \ref{subsec:derive_updated_rule_proof}.
\end{proof}
Note that the iterative process in Theorem \ref{theo:mean_field} is performed in a distributed fashion, where each element $q_i(X_i=\mathbf{x}_{i})$ in $q(X)$ is calculated recursively as given below.
\begin{equation}\label{eq:criteria}
\begin{aligned}
    q_{i}(X_i = \mathbf{x}_{i}) &= \frac{1}{Z_i}\exp\Bigg (-\sum _{j \in \mathcal{N}_i} \sum _{\mathbf{x}_j} q_{j}(\mathbf{x}_{j}) \psi _{ar}(\mathbf{x}_{i},\mathbf{x}_{j}) \\ 
    & - \sum _{\beta \in \mathcal{O}_{i}} \psi _{o}(\mathbf{x}_{i},\mathbf{x}_{\beta}) -\psi _{r}(\mathbf{x}_{i},\mathbf{x}_{r}) \Bigg)
\end{aligned}
\end{equation}
% It is apparent that the iterative calculation of $q_i(X_i=\mathbf{x}_{i})$ in \eqref{eq:criteria} is carried out based on local information from the neighbors of robot $i$. Hence, the iterative process \eqref{eq:criteria} is computationally tractable for a distributed flocking system. Based on the iterative calculation $q_{i}(X_i = \mathbf{x}_{i})$ ($\forall i \in \mathcal{A}$), one is able to obtain an optimal mean-field approximation $q(X)$ of the true distribution $p(X)$. Hence, the inference of the best control is performed using $q(X)$ instead of $p(X)$. 
The iterative calculation of $ q_i(X_i=\mathbf{x}_i) $ in \eqref{eq:criteria} relies on local information from robot $ i $'s neighbors, making it computationally feasible for a distributed flocking system. Through this iterative process, an optimal mean-field approximation $ q(X) $ of the true distribution $ p(X) $ can be obtained. Hence, the best control is inferred using $ q(X) $ rather than $ p(X) $.
In the sequel, $q(X)= \prod _{i} q_i(X_i)$ is employed to infer the optimal control. 

\subsection{Inference of the Optimal Control}\label{subsec:Optimal Control Input}
% The optimal flocking corresponds to the spatial configuration of robots with the highest possibility on the GRF defined in Section \ref{subsec:Iterative Optimization}. The optimal control is inferred by maximizing a posteriori distribution of $q(X)$ in a certain future time instant, which would lead to flocking with minimal potential energies. The future states of each robot are predicted using the dynamic model \eqref{eq:state-space model}. Hence, the optimal control is calculated using both current and future information, so it is inherently a predictive control method. Since the approximated distribution $q(X)$ is fully factorized with  $q(X) = \prod _{i} q_i(X_i)$ and $q_i(X_i)$ given in \eqref{eq:criteria}. The optimal control inference process is fully distributed for robot $i$ by maximizing its local distribution $q_i(X_i)$ at a future time instant. However, the exact prediction of $q_i(X_i)$ is intractable, as $\mathbf{x}_j$ are unclear in the future time. To resolve this issue, it is assumed that neighbors of robot $i$ keep their current input during the whole prediction horizon. 

% Optimal flocking aligns with the spatial configuration of robots that maximizes the GRF probability, as detailed in Section \ref{subsec:Iterative Optimization}. 
The optimal control is determined by maximizing the posterior distribution of $ q(X) $ at a future time, resulting in flocking with minimal potential energies. Future robot states are predicted using the dynamic model \eqref{eq:state-space model}. Hence, the optimal control is calculated using both current and future information, so it is inherently a predictive control method. Since the approximated distribution $ q(X) $ is fully factorized with $ q(X) = \prod_i q_i(X_i) $, the optimal control inference for robot $ i $ is fully distributed, maximizing its local distribution $ q_i(X_i) $ at a future time instant. However, predicting $ q_i(X_i) $ precisely is intractable because future states $ \mathbf{x}_j $ are unclear. To address this, it is assumed that the neighbors of robot $ i $ maintain their current input throughout the prediction horizon.

% The maximum a posteriori for control inference is a nonlinear optimal control problem with very complex objective functions as in \eqref{eq:criteria}, so its exact solution is computationally intractable. Instead, researchers attempted to find an approximate solution by evenly discretizing the input space of robot $i$ in a similar way to the state lattice planner \cite{fernando2020swarming}, as illustrated in Fig. \ref{fig:control inputs}(a). The best control in the discrete input set is determined by using \eqref{eq:criteria} to evaluate its posterior distribution. The performance of the uniform discretization method depends heavily on the subtleness of the input space discretization. A delicate discretization would lead to a large number of inputs to be evaluated, thus resulting in a burgeoning computation cost. 
The maximum a posteriori for control inference is a nonlinear optimal control problem with complex objective functions as in \eqref{eq:criteria}, making an exact solution computationally intractable. Instead, researchers have used an approximate solution by discretizing the input space of robot $ i $ evenly, similar to the state lattice planner \cite{fernando2020swarming}, as shown in Fig. \ref{fig:control inputs}(a). The best control is chosen from the discrete set by evaluating its posterior distribution using \eqref{eq:criteria}. The performance of this method heavily relies on the fineness of the input space discretization. A finer discretization exponentially increases the number of inputs to evaluate, leading to a burgeoning computation cost.

% In our design, the conflict between the discretization subtleness and computational cost is mitigated by introducing a heuristic solution $\mathbf{u}_g$. The input space discretization is biased by the heuristic solution. Instead of discretizing the whole input space evenly, we only need to discretize a neighborhood input space around $\mathbf{u}_g$ as shown in Fig. \ref{fig:control inputs}(c). With $\mathbf{u}_g$, the whole search space would be reduced significantly, thus leading to a delicately discretized input set and low computation cost. 
In our design, the conflict between discretization subtleness and computational cost is mitigated by introducing a heuristic solution $\mathbf{u}_g$. The input space is biased around $\mathbf{u}_g$, allowing us to discretize only a neighborhood of the input space rather than the whole space, as shown in Fig. \ref{fig:control inputs}(c). This approach significantly reduces the search space, enabling a finer discretization and lower computational cost.

% The heuristic solution $\mathbf{u}_g$ is designed based on the close relationship between GRF and potential energies in Section \ref{subsec:Potential Energies}. 
The objective of the predictive control is to minimize the total potential energies by \eqref{eq:Attraction and repulsion}, \eqref{eq:Obstacle repulsion}, \eqref{eq:obstacle avoidance}, \eqref{eq:Attraction of virtual leader}, and \eqref{eq:Velocity cost}. Hence, a natural design of $\mathbf{u}_g$ is given by the combination of the negative gradients of the potential energies in Section \ref{subsec:Potential Energies}. For the potential energies (\ref{eq:Attraction and repulsion}) and (\ref{eq:Attraction of virtual leader}), we have $\mathbf{u}_{gar,i}  = -\nabla _{\mathbf{p}_i} \psi _{ar} $ and $\mathbf{u}_{grp,i} = -\nabla _{\mathbf{p}_i} \psi _{rp}$, so 
\begin{equation}\label{eq:gradient of inter-robot}
\begin{aligned}
    \mathbf{u}_{gar,i} 
    &= \begin{cases}
        \pi \frac{k_a}{r_f}\sin{\left (\pi \frac{\Vert\mathbf{p}_{ij}\Vert}{r_f} \right )}\mathbf{n}_{ij},& \text{if} \ \Vert\mathbf{p}_{ij}\Vert \leq k_tr_f \\
        \pi \frac{k_a}{r_f} \sin{(\pi k_t)} \mathbf{n}_{ij},& \text{otherwise.}
    \end{cases} 
\end{aligned}
\end{equation}
\begin{equation}\label{eq:gradient of robot-leader}
    \mathbf{u}_{grp,i}  = -k_{rp}\exp{(\Vert \mathbf{p}_{ir}\Vert)}\mathbf{n}_{ir}
\end{equation}
where $\mathbf{n}_{ij} = \mathbf{p}_{ij}/\Vert \mathbf{p}_{ij} \Vert$ and $\mathbf{n}_{ir} = \mathbf{p}_{ir}/\Vert \mathbf{p}_{ir} \Vert$. Note that \eqref{eq:gradient of inter-robot} and \eqref{eq:gradient of robot-leader} are implemented for the scenario where $\Vert \mathbf{p}_{ij} \Vert,\Vert \mathbf{p}_{ir} \Vert > 0$. To ensure smooth performance, damping forces $\mathbf{u}_{gav,i} = -k_{av}\mathbf{v}_{ij}$ and $\mathbf{u}_{grv,i} = -k_{rv}^{'}\mathbf{v}_{ir}$ are added to the heuristic solution $\mathbf{u}_{g,i}$ of robot $i$, where $k_{av}, k_{rv}^{'} > 0$ denote the damping gain. In an obstacle-free environment, the heuristic solution $\mathbf{u}_{g,i}$ is denoted by $\mathbf{u}_{g,i}^0$.
\begin{equation}\label{eq:The first term of initial solution}
    \mathbf{u}_{g,i}^0 = \sum _{\mathcal{N}} \mathbf{u}_{gar,i} + \sum _{\mathcal{N}}\mathbf{u}_{gav,i} + \mathbf{u}_{grp,i} + \mathbf{u}_{grv,i}
\end{equation}
In an obstacle-free environment, $\mathbf{u}_{g,i}^0$ is possible to ensure a certain convergent flocking behavior as stated in Theorem \ref{theo:initial_solution}. 
\begin{figure}[tbp]
    \centering
    \includegraphics[width = 1\linewidth]{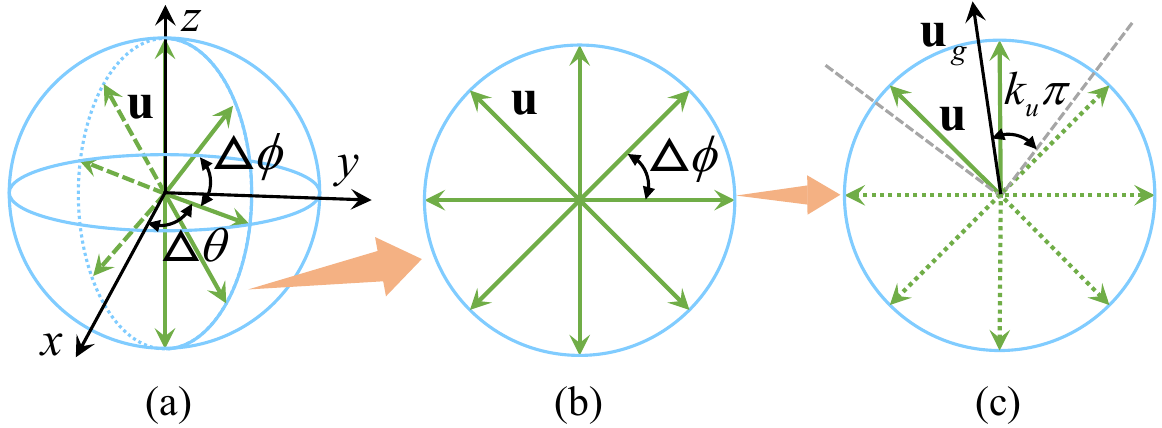}
    \caption{(a) Control discretization (green arrow). (b) Uniform discretization. (c) Biased local discretization (solid green arrow) based on $\mathbf{u}_g$ which has a cone-like shape in space.}
    \label{fig:control inputs}
\end{figure}
\begin{theorem}\label{theo:initial_solution}
    Consider $N$ robots with the dynamics (\ref{eq:state-space model}) in an obstacle-free environment. If $\mathbf{v}_r$ is a constant, $V(0)$ is finite, and $\mathcal{G}$ is connected, the velocities of all the robots will asymptotically converge to $\mathbf{v}_r$ by the control law (\ref{eq:The first term of initial solution}).
\end{theorem}
\begin{proof}
    See Appendix \ref{subsec:initial_solution_proof}.
\end{proof}
\begin{remark}
    The conditions of Theorem \ref{theo:initial_solution} are easily satisfied. There always exists $0 \leq \psi _{ar} \leq 2k_a$. Also, $\psi _{rp}$ is finite if robot $i \in \mathcal{A}$ is located at a finite distance to the target position. Hence, $V(0)$ is always finite. If $\Vert \mathbf{p}_{ij} \Vert \leq r_s$, $\mathcal{G}$ is always connected. In fact, even if the graph is split, it can be re-connected by the control law (\ref{eq:The first term of initial solution}).
\end{remark}
\begin{proposition}\label{prop:proposition_1}
    Consider the same conditions as in theorem \ref{theo:initial_solution} other than that $\mathcal{G}$ is not connected. Then, the graph $\mathcal{G}$ of a flock can be re-connected by the control law (\ref{eq:The first term of initial solution}).
\end{proposition}
\begin{proof}
    See Appendix \ref{subsec:graph_connected_proof}.
\end{proof}
 
To incorporate the collision-avoidance behavior in congested environments, we introduce the following terms.
\begin{equation}\label{eq:heuristic obstacle avoidance}
\left\{\begin{array}{cl}
\mathbf{u}_{gor,i} &= -\nabla _{\mathbf{p}_i} \psi _{or}\\
\mathbf{u}_{gob,i} &= k_{ob}\mathbf{v}_{ob} / \Vert\mathbf{v}_{ob}\Vert = k_{ob} \mathbf{R}_{GL}^T \mathbf{R}_{L}^T\left [\begin{array}{c}
        1,\ 0,\ 0
    \end{array}
    \right ]^T
    \end{array}\right.
\end{equation}
where $k_{ob} > 0$ is the bypass speed gain, which determines the degree of bypass smoothness, $\mathbf{R}_{GL} = \underline{\mathbf{e}}_L\cdot \underline{\mathbf{e}}_G^T$ with $\underline{\mathbf{e}}_L = \left [\mathbf{L}_1, \mathbf{L}_2,\mathbf{L}_3 \right ]^T$, $\mathbf{L}_1 = -\mathbf{p}_{ij} /\Vert\mathbf{p}_{ij}\Vert$, $\mathbf{L}_3 = -\mathbf{p}_{i\beta} \times \mathbf{v}_{i\beta}/\Vert\mathbf{p}_{i\beta} \times \mathbf{v}_{i\beta}\Vert$, $\mathbf{L}_2 = \mathbf{L}_3 \times \mathbf{L}_1$ and $ \underline{\mathbf{e}}_G = \left [\mathbf{G}_1, \mathbf{G}_2, \mathbf{G}_3 \right ]^T$, $ \mathbf{G}_1 = \left [1, 0, 0 \right ]^T$, $ \mathbf{G}_2 = \left [0, 1, 0 \right ]^T$, $ \mathbf{G}_3 = \left [0, 0, 1 \right ]^T$, and $\mathbf{R}_L = \left [\cos{\theta _{\mathrm{III}}} \ \sin{\theta _{\mathrm{III}}} \ 0; -\sin{\theta _{\mathrm{III}}} \ \cos{\theta _{\mathrm{III}}} \right.$ $\left. \ 0; 0 \ 0 \ 1 \right ] $. 

The heuristic solution related to collision avoidance is $\mathbf{u}_{g,i}^1 = \sum _{\mathcal{O}} \mathbf{u}_{gor,i} + \sum _{\mathcal{O}} \mathbf{u}_{gob,i}$. Hence, the overall heuristic solution $\mathbf{u}_{g,i}$ of the robot $i$ is given by
\begin{equation}\label{eq:Initial solution chapter5}
    \mathbf{u}_{g,i} = \mathbf{u}_{g,i}^0 + \mathbf{u}_{g,i}^1
\end{equation}
Based on \eqref{eq:Initial solution chapter5}, the discretization results are 
\begin{align}\label{eq:rectify control space}
    \mathcal{U}_i = & \left \{\mathbf{u}_i|\mathbf{u}_i = u_t\mathbf{e}(k_1\triangle \theta,k_2\triangle \phi),k_1 \in [0,k_{\theta}), k_2 \in [0,  \right. \nonumber \\ 
    & \left. k_{\phi}), \angle(\mathbf{u},\mathbf{u}_{g,i}) \leq k_u\pi ,k_1,k_2,k_{\theta},k_{\phi} \in \mathbb{Z} \right \}
\end{align}
where $\mathcal{U}_i$ is the set of input candidates of robot $i$ to be evaluated, $u_t = \{u_{max}/n_{u},\ldots,u_{max}\}$ with $n_u \in \mathbb{Z}$ is an arithmetic progression, $\mathbf{e}(\theta,\phi) = [\cos{\phi} \cos{\theta} \ \cos{\phi} \sin{\theta} \ \sin{\phi}]^T$, $\angle (\mathbf{u},\mathbf{u}_{g,i}) = \arccos{(\mathbf{u}\cdot\mathbf{u}_{g,i}}/\Vert\mathbf{u}\Vert\cdot\Vert\mathbf{u}_{g,i}\Vert)$, and $ \triangle \theta$ and $ \triangle \phi$ are shown in Fig. \ref{fig:control inputs}. And $k_{\theta}$ and $k_{\phi}$ are discretization coefficients which are chosen such that $k_{\theta} \triangle \theta = k_{\phi} \triangle \phi = 2 \pi$. In addition, $0 < k_u < 1$ 
\begin{algorithm}[tbp]
    \caption{Heuristic predictive flocking control}\label{algorithm:Predictive control with initial solution}
    \begin{algorithmic}[1] % 加上 [1] 表示有序号
        \For {$ h \in \mathcal{N}_i \cup \{i\}$}
            \State Compute $\mathbf{u}_{g,h}$ at time instance $t_k$
            \State Compute $\mathcal{U}_{h,t_k}$ based on $\mathbf{u}_{g,h}$ at time instance $t_k$   
            \State Compute $\mathcal{X}_{h,t_k + H}$ based on dynamics model
            \State Initialization: $q_h(\mathbf{x}_{h,t_k + H}) \longleftarrow 1/\vert \mathcal{X}_{h,t_k + H} \vert$
        \EndFor
        \While {$q_{old}(X) \neq q(X)$}
            \State $q_{old}(X) \longleftarrow q(X)$
            \For {$\mathbf{x}_{h,t_k + H} \in \mathcal{X}_{h,t_k + H}$}
                \State 
                $\begin{aligned}
                    \psi _{two} \longleftarrow &  \sum _{j \in \mathcal{N}_i \cup \{i\} / \{h\} } \sum _{\mathbf{x}_{j,t_k + H}} q_j(\mathbf{x}_{j,t_k + H}) \times \\
                    & \psi _{ar} (\mathbf{x}_{h,t_k + H}, \mathbf{x}_{j,t_k + H})
                \end{aligned}$
                \State 
                $\begin{aligned}
                    \psi _{single} \longleftarrow & \sum _{\beta \in \mathcal{O}_h}\psi _{o}(\mathbf{x}_{h,t_k + H}, \mathbf{x}_{\beta}) \\
                    &+ \psi _{r}(\mathbf{x}_{h,t_k + H}, \mathbf{x}_{r})
                \end{aligned}$
                \State $q_h(\mathbf{x}_{h,t_k + H}) \longleftarrow \exp{( - \psi _{two} - \psi _{single})}$
                \State Normalize $q_h(X_h)$ to sum to one
            \EndFor
        \EndWhile 
        \State Find $\mathbf{x}_{i,t_k + H}^*$ so that $\arg_{\mathbf{x}_{i,t_k + H}^*} \max q_i(\mathbf{x}_{i,t_k + H})$
        \State Get $\mathbf{u}_i^*$ from $\mathcal{U}_{i,t_k}$ based on $\mathbf{x}_{i,t_k + H}^*$
        \State \Return $\mathbf{u}_i^*$
    \end{algorithmic} 
\end{algorithm}
\begin{figure*}
    \centering
    \includegraphics[width = 1\linewidth]{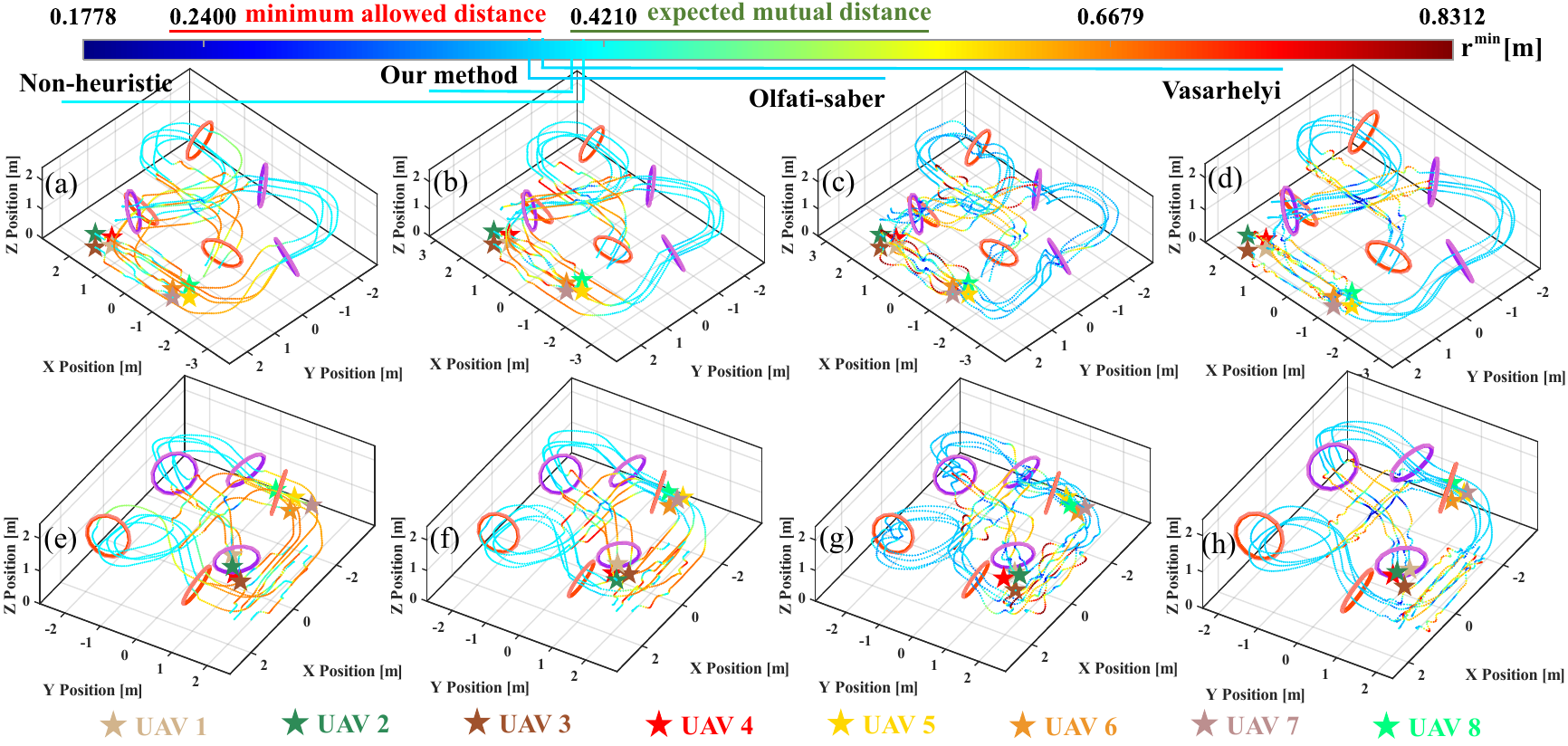}
    \caption{Simulation trajectories. (a)(e):The non-heuristic method;(b)(f): Our method; (c)(g): Olfati-saber's method; and (d)(h): V\'{a}s\'{a}rhelyi's method. The first row shows the trajectories from a view angle of $az = -135^{\circ}$ and $el = 42^{\circ}$, while the second row shows the trajectories from a view angle of $az = 121^{\circ}$ and $el = 38^{\circ}$. In the colorbar, the light blue lines leading from different positions indicate the distance between the robot and its nearest neighbor when the four methods form a stable motion (\emph{i.e.}, the robot has the same velocity as the reference state $\mathbf{v}_r$). Obviously, the distance between the UAVs is closer to the expected value $r_f = 0.421 \ \mathrm{m}$ in the absence of obstacles when using the first two methods(non-heuristic $0.417 \ \mathrm{m}$, our method $0.412 \ \mathrm{m}$, Olfati-saber $0.3912 \ \mathrm{m}$, V\'{a}s\'{a}rhelyi $0.3978 \ \mathrm{m}$). The UAV trajectories by our method  are smoother in comparison with the other methods.}
    \label{fig:simulation trajectory}
\end{figure*}

Based on $\mathcal{U}_i$, the predicted states of robot $i$ are denoted as $\mathcal{X}_{i,t_k + H} = \{ \mathbf{x}_i(t_k + H) \vert \mathbf{x}_i(t_k + H) = \mathbf{A}\mathbf{x}_i(t_k) + \mathbf{B}\mathbf{u}_i(t_k) \}$, $\mathbf{u}_i(t_k) \in \mathcal{U}_{i,t_k}$, $i \in \mathcal{A}$, where $H$ is the prediction horizion, and $\mathcal{U}_{i,t_k}$ is the input space of robot $i$ at the time instant $t_k$. A posteriori distribution of each predicted state in $\mathcal{X}_{i,t_k + H}$ is, therefore, calculated using \eqref{eq:criteria}. The optimal control $\mathbf{u}_i^*(t_k)$ is determined by the candidate control input $\mathbf{u}_i(t_k)\in\mathcal{U}_{i,t_k}$ leading to MAP distribution on $\mathbf{x}_i(t_k + H)$. The whole method of robot $i$ is summarized in Algorithm \ref{algorithm:Predictive control with initial solution}.
\begin{figure}
    \centering
    \includegraphics[width = 1\linewidth]{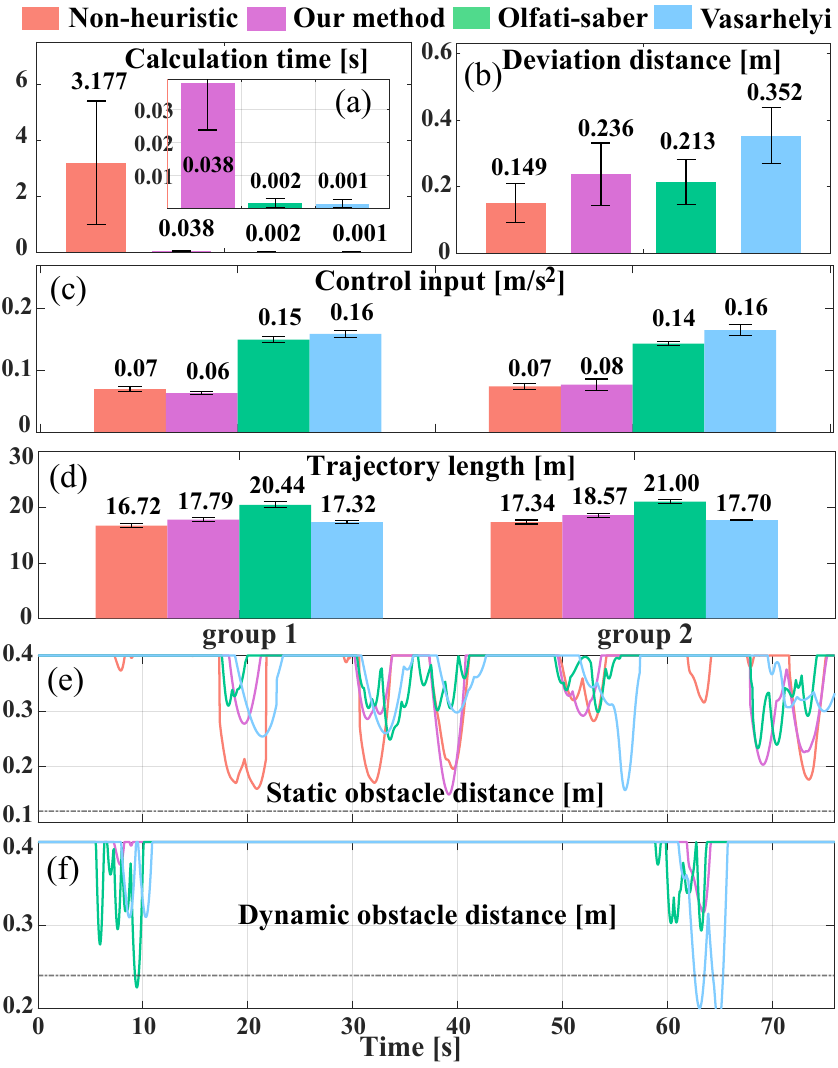}
    \caption{Performance comparison. The two predictive methods have obvious advantages in most aspects, including control efficiency and collision avoidance. The computation efficiency of our method is greatly improved as compared to that of the non-heuristic predictive method.}
    \label{fig:simulation metrics}
\end{figure}
	
\section{Simulations and Experiments}\label{sec:Experiments}
Numerical simulation and real-world experiments are conducted to validate the efficiency of the proposed algorithm in this section. Some metrics are introduced to evaluation
algorithm performance. 

The \emph{calculation time metric} \label{eq:calculation time metric} evaluates how much time the computation of the algorithm took on average at each step, which is calculated by $t_{cal}^{avg} = \triangle t/T\sum _{t_k = 0}^{T/\triangle t}t(t_k)$, where $T$ is the total running time.  A smaller $t_{cal}^{avg}$ implies that the corresponding method is more time-efficient.

The \emph{deviation metric} calculates the distance of the flock center to the reference position $\mathbf{p}_r$. It is denoted as $r_{dev} = \sum _{i \in \mathcal{A}} \Vert \mathbf{p}_{ir} \Vert/N \in \mathbb{R}$. A smaller deviation indicates that the corresponding method has better tracking performance.

The \emph{control input metric} evaluates the input efficiency of each robot, which is calculated by $u^{avg} = \triangle t/T \sum_{t_k=0}^{T/\triangle t} ||\mathbf{u}^*(t_k)||$\label{eq:mean control input metric}. A smaller $u^{avg}$ implies that it is more energy efficient.

The \emph{trajectory length metric} evaluates the distance traveled by robots. A short trajectory length is preferable in general, which is calculated by $L = \sum_{\substack{t_k=0 }} ^{T/\triangle t - 1}||\mathbf{p}(t_k + \triangle t) - \mathbf{p}(t_k)||$\label{eq:trajectory length metric}. 

The \emph{distance metric} evaluates the minimum distance of the robot from its neighboring robots and the minimum distance from the obstacle, the former is calculated by $r^{min} = \min \{r_{ij}^{min},r_{i\beta}^{min} \} \in \mathbb{R}^{N}$ and the latter is $r_{static}^{min} = \min\{r_{i\beta}^{min} \vert \Vert \mathbf{v}_{\beta} \Vert = 0 \} \in \mathbb{R}$ and $r_{dynamic}^{min} = \min\{r_{i\beta}^{min} \vert  \Vert \mathbf{v}_{\beta} \Vert \geq 0 \} \in \mathbb{R}$, where $r_{ij}^{min} = \min\{\Vert \mathbf{p}_{ij} \Vert \vert  j \in \mathcal{N}_i\}$, $r_{i\beta}^{min} = \min\{\Vert\mathbf{p}_{i\beta}\Vert \vert \beta \in \mathcal{O}_i\}$. The minimum distance should be guaranteed to meet safety constraints. A larger $r^{min}$ indicates an algorithm is more secure.

\subsection{Numerical Simulations}\label{subsec:Simulation Analysis}
In the simulation, we consider two independent groups of UAVs flying in one confined environment with multiple openings. Each group has four UAVs, so there are eight UAVs in total in the simulation. There are no communications for UAVs from different groups, so UAVs from a different group are taken as moving obstacles. Each group should follow their respective reference state $\mathbf{x}_r$ and pass through the purple or orange door frames as shown in Fig. \ref{fig:simulation trajectory}. Note that $r_f$ is set to $k_nr_f$ ($k_n > 1$) when the robot traverses dense obstacles. Such a simulation environment is congested, competitive, and challenging for all UAVs. The proposed heuristic predictive control algorithm is compared with a non-heuristic version and two existing popular methods, namely Olfati-saber's method \cite{olfati2006flocking} and the V\'{a}s\'{a}rhelyi's method \cite{vasarhelyi2018optimized}. This non-heuristic version does not have $\mathbf{u}_g$ with $k_u = 1$, but all other implementation details, such as parameter settings, are consistent with our heuristic design. Choose $r_c = 0.12 \ \mathrm{m}$, $r_{\beta} = 0.12 \ \mathrm{m}$, $v_{max} = 0.3 \ \mathrm{m/s}$, $u_{max} = 0.4 \ \mathrm{m/s^2}$, $\triangle t = 0.05 \ \mathrm{s}$, $H = 0.15 \ \mathrm{s}$. All the necessary parameters are summarized in Table \ref{tab:Algorithm parameters setup}. All algorithms are implemented in MATLAB on a desktop computer with Intel i7-8550U 1.8GHz CPU and 16GB memory. The UAV trajectories from different methods are illustrated in Fig. \ref{fig:simulation trajectory}. The GRF-based optimal control---both the heuristic and non-heuristic versions---has a better performance in keeping expected mutual distances among UAVs than the other two methods as shown in Fig. \ref{fig:simulation trajectory}.
\begin{figure}
    \centering
    \includegraphics[width = 1\linewidth]{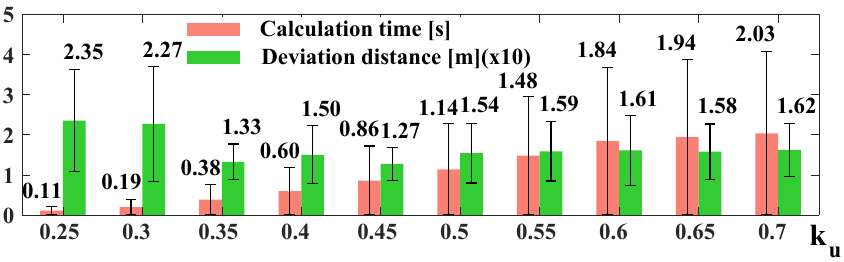}
    \caption{The influence of different $k_u$ on $t_{cal}^{avg}$ and $r_{dev}$. The calculation time $t_{cal}^{avg}$ increases as $k_u$ grows. The deviation distance $r_{dev}$ decreases with $k_u$ when $k_u < 0.35$, but it stays almost the same, when $k_u > 0.35$.}
    \label{fig:simulation statistic}
\end{figure}

A more detailed comparison is illustrated in Fig. \ref{fig:simulation metrics}.
As shown in Fig. \ref{fig:simulation metrics}(a), the calculation time $t_{cal}^{avg}$  of our method is reduced significantly in comparison with the non-heuristic version. Our method is $83$ times faster than the non-heuristic one. Our method only needs to consider the control inputs in the neighborhood around the heuristic solution $\mathbf{u}_g$, so it is possible to avoid evaluating a large number of non-necessary inputs as in the non-heuristic one.  The computational burden of iterative optimization is, therefore, reduced significantly. However, though the search space for inputs is reduced, the control performance keeps almost the same as the non-heuristic one, from the tracking performance $r_{dev}$, the control input $u^{avg}$, and trajectory length $L$ perspectives, as shown in Fig. \ref{fig:simulation metrics}(b)-(d). It demonstrates that $\mathbf{u}_g$ can provide reasonable guidance for robot flocking. The introduction of the heuristic solution $\mathbf{u}_g$ can dramatically reduce the computation cost with almost no sacrifice of the control performance. 
 
It should also be pointed out that our method has advantages over the non-optimal solutions, such as Olfati-saber's method \cite{olfati2006flocking} and the V\'{a}s\'{a}rhelyi's method \cite{vasarhelyi2018optimized}, from a collision-avoidance perspective, as shown in Fig. \ref{fig:simulation metrics}(e)(f). It is fundamentally due to both the finely designed potential energy for collision avoidance and the heuristic solution.  For example, the direction potential $\psi _{od}$ and repulsion potential $\psi _{or}$ ensure that the algorithm has sufficient collision avoidance ability to meet the safety constraints (Fig. \ref{fig:simulation metrics}(e)(f)). The heuristic solution contains damping terms $\mathbf{u}_{gav}$ and $\mathbf{u}_{grv}$ that can reduce the motion oscillations (Fig. \ref{fig:simulation metrics}(d), Fig. \ref{fig:simulation trajectory}). 

By contrast, the non-optimal strategies--- Olfati-saber's method \cite{olfati2006flocking} and the V\'{a}s\'{a}rhelyi's method \cite{vasarhelyi2018optimized},---don't work well in congested challenging scenarios. In essence, they both use the idea of the artificial potential function, so they make reactions based on the current state, making it impossible to foresee the motion of other robots in flocking. Hence, it is impossible to find an optimal solution for either Olfati-saber's method \cite{olfati2006flocking} or the V\'{a}s\'{a}rhelyi's method \cite{vasarhelyi2018optimized}. Furthermore, the resulting performance of the two non-optimal strategies is sensitive to the parameters, making it difficult to balance the performance of all metrics. Our method is an optimal solution that is calculated based on both the current and predicted states of all robots, so its input can ensure improved flocking performance over the existing non-optimal solutions as demonstrated in Fig. \ref{fig:simulation metrics} (b)-(f). It should also be noted that the time efficiency and tracking performance of our method are deeply affected by $k_u$ that dominates the size of the input space for discretization as shown in Fig. \ref{fig:simulation statistic}.
\begin{table}
    \centering
    \caption{Algorithm parameters}
    \begin{threeparttable}
    \setlength{\tabcolsep}{0.8em}{
    \renewcommand\arraystretch{1.3}
    \begin{tabular}{*{6}{l} }
        \toprule
         \multicolumn{2}{c}{Our method} & \multicolumn{2}{c}{Olfati-saber \cite{olfati2006flocking}} & \multicolumn{2}{c}{V\'{a}s\'{a}rhelyi \cite{vasarhelyi2018optimized}} \\
        \midrule
          $r_f(\mathrm{m})$ & 0.421 & $d_{\alpha}(\mathrm{m})$ & 0.421 & $r_0^{rep}(\mathrm{m})$ & 0.421 \\
          $r_s(\mathrm{m})$ & 0.4631 & $r_{\alpha}(\mathrm{m})$ & 0.4631 & $r^{cluster}(\mathrm{m})$ & 0.4631 \\
          $k_a$ & 12 & $d_{\beta}(\mathrm{m})$ & 0.421 & $p^{rep}$ & 3 \\
          $k_t$ & 2 & $r_{\alpha}(\mathrm{m})$ & 0.4631 & $r_0^{frict}(\mathrm{m})$ & 0.3 \\
          $k_n$ & 1.6 & $k_n$ & 1.6 & $k_n$ & 1.6 \\
          $k_{or}$ & 20 & $a$ & 8 & $C^{firct}$ & 0.25 \\
          $k_{od}$ & 10 & $b$ & 10 & $v^{frict}(\mathrm{m/s})$ & 0.1 \\
          $k_{\delta}$ & 0.5 & $h$ & 0.4 & $p^{frict}$ & 0.3 \\
          $\delta$ & 0.3 & $c_1^{\alpha}$ & 1 & $a^{frict}(\mathrm{m/s^2})$ & 0.2 \\
          $k_{\rho}$ & 2 & $c_2^{\alpha}$ & 1 & $r_0^{shill}(\mathrm{m})$ & 0.1 \\
          $k_{rp}$ & 5 & $c_1^{\beta}$ & 1.5 & $v^{shill}(\mathrm{m/s})$ & 0.1 \\
          $k_{rv}$ & 15 & $c_2^{\beta}$ & 1 & $p^{shill}$ & 0.3 \\
          $k_{av}$ & 40 & $c_1^{\gamma}$ & 0.4 & $a^{shill}(\mathrm{m/s^2})$ & 0.2 \\
          $k_{rv}^{'}$ & 0.1 & \multicolumn{2}{c}{Our method} & \multicolumn{2}{c}{Our method} \\
          \cline{3-6}
          $k_{ob}$ & 10 & $k_{\theta}$ & 12 & $k_{u}$ & 0.2 \\
          $n_u$ & 2 & $k_{\phi}$ & 12  \\
        \bottomrule
    \end{tabular}
    }
    \end{threeparttable}
    \label{tab:Algorithm parameters setup}
\end{table}

Additionally, we conducted scalability experiments using the challenging scenario depicted in Fig \ref{fig:simulation trajectory}, which includes open/narrow areas and static/dynamic obstacles, making it sufficiently representative. We tested with two sets of UAVs, increasing the number to 2, 4, 6, and 8. The setup for $n=2$ is the same as for $n=4$, shown in Fig \ref{fig:simulation trajectory}. The configurations for $n=6$ and $n=8$ have some different parameters: 
$n=6$: $k_n = 1.65$, $k_{rp} = 25$, $k_{rv} = 17$; $n=8$: $k_a = 18$, $k_n = 1.7$, $k_{or} = 18$, $k_{od} = 12$, $k_{rp} = 26$, $k_{rv} = 20$, $k_{ob} = 18$. The flocking performance is shown in Table \ref{tab:control_performence_with_variable_number}, where $n$ denotes the number of UAVs in each group. Since the simulation runs on a single computer, $t_{cal}^{avg}$ refers to the total time taken by all robots to compute the control commands at each moment. It naturally increases with the number of robots, but it is an acceptable computational burden. In conclusion, due to the introduction of the heuristic solution $\mathbf{u}_g$ and the multi-layer collision avoidance mechanism, the overall performance of the swarm does not significantly decline as the number of robots increases. The scalability of our method is well behaved.
\begin{table}
    \centering
    \caption{Control performence for defferent numbers of robots}
    \begin{threeparttable}
    \setlength{\tabcolsep}{1.2em}{
    \renewcommand\arraystretch{1.5}
    \begin{tabular}{*{5}{l} }
        \toprule
        & $n = 2$ & $n = 4$ & $n = 6$ & $n = 8$ \\
        \midrule
        $t_{cal}^{avg} (\mathrm{s})$ & 0.0116 & 0.0375 & 0.0563 & 0.0748 \\
        $r_{static}^{min} (\mathrm{m})$ & 0.2297 & 0.1600 & 0.1550 & 0.1918  \\
        $r_{dynamic}^{min} (\mathrm{m})$ & 0.3778 & 0.2793 & 0.2809 & 0.2574 \\
        $r_{dev}^{avg} (\mathrm{m})$ & 0.2198 & 0.2359 & 0.2187 & 0.2321 \\
        $u^{avg} (\mathrm{m/s^2})$ & 0.0616 & 0.0697 & 0.0758 & 0.0698 \\
        $L^{avg} (\mathrm{m})$ & 18.0702 & 18.1798 & 18.6360 & 18.7603 \\
        \bottomrule
    \end{tabular}
    }
    \end{threeparttable}
    \label{tab:control_performence_with_variable_number}
\end{table}

\subsection{Real Experiment}\label{subsec:Real Experiment}
The experimental scenario is a simplified version of the simulation to validate the performance of the proposed method in real physical systems. We implement our algorithm on four DJI Tello drones that are divided into two groups. Each group with two drones is required to pass through two door frames as shown in Fig. \ref{fig:real_trajectory}. The safety radius of the tello drone is $0.1 \ \mathrm{m}$, so we have $r_c = r_{\beta} = 0.1 \ \mathrm{m}$. Choose $v_{max} = 0.24 \ \mathrm{m/s}$, $u_{max} = 0.75 \ \mathrm{m/s^2}$, $\triangle t = 0.02 \ \mathrm{s}$, $H = 0.06 \ \mathrm{s}$. In the experiment, we choose $r_f=0.7033$ $\mathrm{m}$, $r_s = 0.7736$ $\mathrm{m}$, $k_u=0.3$, $k_n=1.2$, $k_a=24$, and $k_{rp}=28$, while other parameters are the same as those in the first column of Table \ref{tab:Algorithm parameters setup}. 

The experiment results are provided in Fig. \ref{fig:snapshots_of_experiment}-\ref{fig:real_metric_static_dynamic_obstacle}. 
It is worth noting that the Olfati-saber's and the V\'{a}s\'{a}rhelyi's methods failed to be implemented due to Tello's internal dynamics, which generates around $0.3 \sim 0.4 \ \mathrm{s}$ time delay, as well as the environment congestion. The non-heuristic version failed to be implemented because it can not match the real-time requirement. As a result, real-world experimental comparisons cannot be conducted like the way in simulation. We are only capable of drawing the robot-obstacle distance. In three-dimensional space, two swarms of drones meet at the same altitude, at which point the drone performs a collision avoidance maneuver to effectively avoid the dynamic obstacles as illustrated in Fig. \ref{fig:snapshots_of_experiment}(c). As shown in Fig. \ref{fig:real_metric_static_dynamic_obstacle}, the two groups of drones can safely move through the two openings by the proposed algorithm with a considerable safety margin.
\begin{figure}[t]
    \centering
    \includegraphics[width = 1\linewidth]{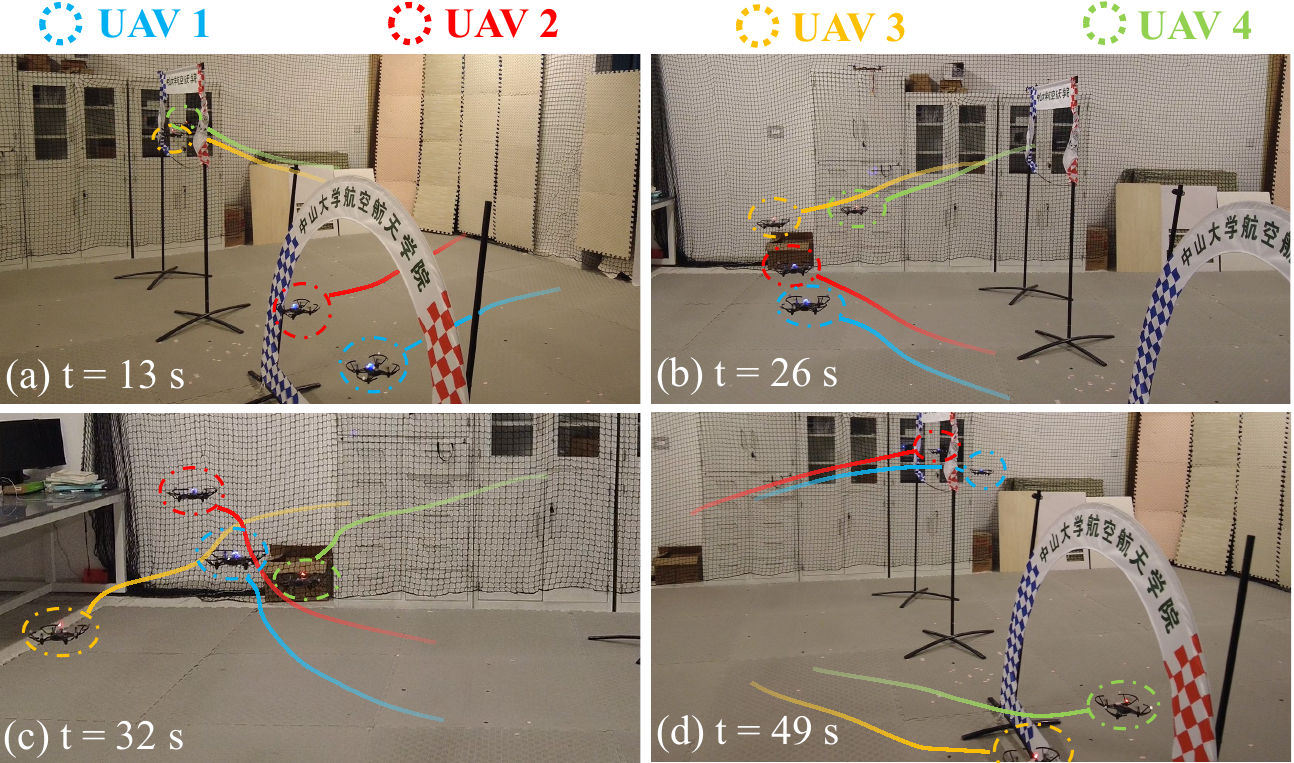}
    \caption{Snapshots of the experiment. The collision avoidance maneuvers occur at $t = 26 \ \mathrm{s}$ - $32 \ \mathrm{s}$.}
    \label{fig:snapshots_of_experiment}
\end{figure}
\begin{figure}[t]
    \centering
    \includegraphics[width = 1\linewidth]{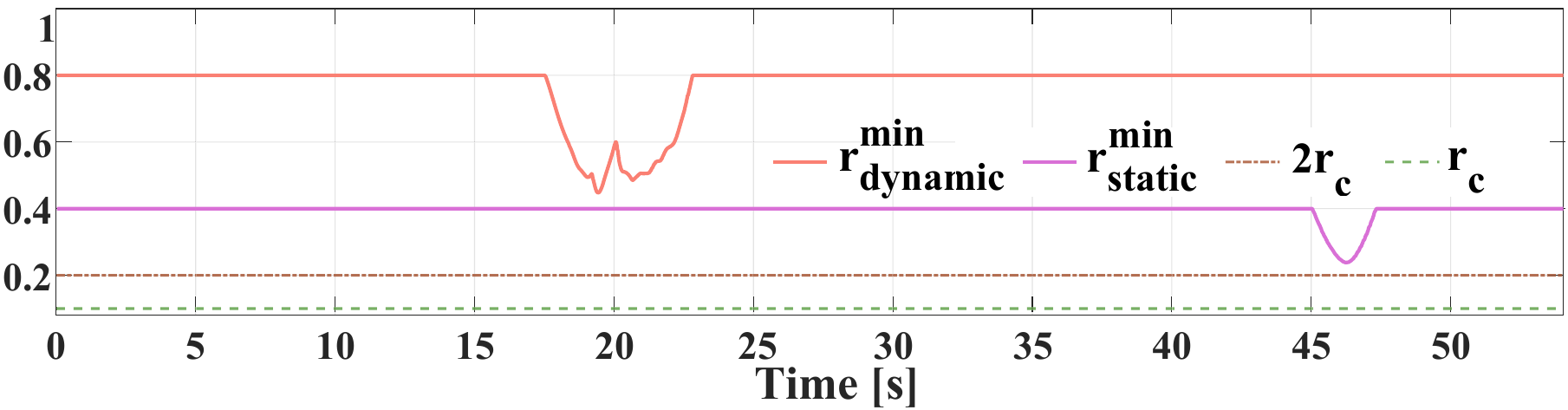}
    \caption{ Distance metrics $r_{static}^{min}$ and $r_{dynamic}^{min}$. Only the cases of $r_{dynamic}^{min} \leq 0.8$ $\mathrm{m}$, $r_{static}^{min} \leq 0.4$ $\mathrm{m}$ are displayed. The minimum allowed distance for $r_{dynamic}^{min}$ is 0.2 m.}
    \label{fig:real_metric_static_dynamic_obstacle}
\end{figure}
	
\section{Conclusions} \label{sec:Conclusion}
In this paper, a heuristic predictive control solution was developed for multi-robot flocking in dynamic challenging environments. In the proposed work, the GRF was employed to characterize the spatial interactions among robots in flocking. Based on the GRF, the control objective was transformed into finding a MAP distribution of a robot flock at a certain future time instant. The MAP distribution corresponds to the spatial configuration with minimum energies. The GRF-based optimal control was resolved via control space discretization. A heuristic solution was introduced as an initial guess for the local discretization of control space. The numerical simulations demonstrated the introduction of the heuristic solution could significantly increase the computation efficiency with little sacrifice of control performance. In addition, a new collision avoidance mechanism was designed, which could result in better collision avoidance performance as shown in the simulations. Theoretical analysis was provided to demonstrate the feasibility of the heuristic solution and convergence of the proposed algorithm.  Eventually, both simulations and experiments were performed to verify the efficiency of the proposed design. 

However, there are still some shortcomings, the ideal model and perception accuracy are assumed, and the algorithm parameters are hand-tuned, which are not conducive to the practical implementation. These are our future research directions.

\appendix

\subsection{Proof of Theorem 1}\label{subsec:derive_updated_rule_proof}
\emph{Proof}: The goal is to seek an approximate distribution $q(X)$ $ = \prod _{i} q_i(X_i)$ which is close to the true distribution $p(X)$ in terms of the KL-divergence $\mathbf{D}(q|p)$. But $\mathbf{D}(q|p)$ cannot be directly minimized due to the fact that $p(X)$ needs to be approximated. Note that $\ln p(X) = \sum _{\mathcal{Q} \in \mathcal{C}} \ln \varPsi _{\mathcal{Q}} (X_\mathcal{Q}) - \ln Z$ and the KL-divergence can be denoted as 
\begin{equation}
    \mathbf{D}(q|p) = \mathbb{E}_{q(X)}\left [\ln q(X) \right] - \mathbb{E} _{q(X)}\left [\sum _{\mathcal{Q} \in \mathcal{C}} \ln \varPsi _{\mathcal{Q}} (X_\mathcal{Q}) \right ] + \ln Z
\end{equation}
where $\mathbb{E}_{q(X)}[\cdot]$ is expectation operation over distribution $q(X)$. Importantly, the term $\ln Z$ does not depend on $q(X)$. Let $F(p,q) = \ln Z - \mathbf{D}(q|p)$, so $\arg_{q(X)} \min \mathbf{D}(q|p)$ is converted to $\arg_{q(X)} \max F(p,q)$. As shown below, $F(p,q)$ can be optimized efficiently by exploiting the structure of $p(X)$. In the following, the Lagrange multiplier method is used to derive a characterization of the stationary points of $F(p,q)$.

Consider the restriction of the objective $F(p,q)$ to those terms that involve $q_i(X_i)$. According to $\sum _{\mathbf{x}_i}q_i(\mathbf{x}_i) = 1$, the Lagrangian can be defined as 
\begin{equation}
    \begin{aligned}
    L_i(p,q) &= \sum _{\mathcal{Q} \in \mathcal{C}}\mathbb{E}_{q(X)}[ \ln \varPsi _{\mathcal{Q}} (X_\mathcal{Q}) ] - \mathbf{H}_{q(X)}(X_i) \\
        &+ \lambda  (\sum _{\mathbf{x}_i}q_i(\mathbf{x}_i) - 1 )
    \end{aligned}
\end{equation}
where $\mathbf{H}_{q(X)}(X_i) = \sum _{\mathbf{x}_i}q_i(\mathbf{x}_i)\ln q_i(\mathbf{x}_i) + \sum _{\mathbf{x}_j \in \bm{X}_{-i}}q_j(\mathbf{x}_j)$ $\ln q_j(\mathbf{x}_j)$, and $\bm{X}_{-i} = X \setminus \{X_i\}$. The derivative with respect to $q_i(\mathbf{x}_i)$ is 
\begin{equation}
    \frac{\partial L_i}{\partial q_i(\mathbf{x}_i)} = \sum _{\mathcal{Q} \in \mathcal{C}} \mathbb{E}_{q(X)}[ \ln \varPsi _{\mathcal{Q}} (X_\mathcal{Q})|\mathbf{x}_i ] - \ln q_i(\mathbf{x}_i) -1 + \lambda
\end{equation}
Then, set it to $0$ and rearrange the terms, we have
$\ln q_i(\mathbf{x}_i) = \lambda - 1 + \sum _{\mathcal{Q}\in \mathcal{C}} \mathbb{E}_{q(X)}[ \ln \varPsi _{\mathcal{Q}} (X_\mathcal{Q})|\mathbf{x}_i ]$.
Take exponents of both sides and renormalize. Because $\lambda$ is a constant relative to $q_i(X_i)$, it will drop out in the renormalization. Therefore, we have
\begin{equation}
    q_i(\mathbf{x}_i) = \frac{1}{Z_i}\exp{\left ( \sum _{\mathcal{Q}\in \mathcal{C}} \mathbb{E}_{q(X)}[ \ln \varPsi _{\mathcal{Q}} (X_\mathcal{Q})|\mathbf{x}_i ] \right )}
\end{equation}
Note that if $X_i \notin X_{\mathcal{Q}}$, there exists $\mathbb{E}_{q(X)}[ \ln \varPsi _{\mathcal{Q}} (X_\mathcal{Q})|\mathbf{x}_i ] = \mathbb{E}_{q(X)}[ \ln \varPsi _{\mathcal{Q}} (X_\mathcal{Q})]$. The expectation terms on the factors which do not contain $X_i$ are independent of $X_i$'s value, they will also be absorbed into $Z_i$. Thus, we have
\begin{equation}\label{eq:qi_xi}
    q_i(\mathbf{x}_i) = \frac{1}{Z_i}\exp{\left ( \sum _{\mathcal{Q}_i} \mathbb{E}_{q'(\mathbf{X}')}[ \ln \varPsi _{\mathcal{Q}_i} (\mathbf{X}',\mathbf{x}_i) ] \right )}
\end{equation}
where $\mathbf{X}' = X_{\mathcal{Q}_i} \setminus \{X_i\}$ and $\mathcal{Q}_i$ is the clique containing $X_i$. Formula (\ref{eq:qi_xi}) is a stationary point of $F(p,q)$. And note that $\sum _{\mathcal{Q} \in \mathcal{C}} \mathbb{E}_{q(X)}\left[ \ln \varPsi _{\mathcal{Q}} (X_\mathcal{Q}) \right]$ is linear in $q_i(X_i)$, $-\mathbf{H}_{q(X)}(X_i)$ is a concave function in $q_i(X_i)$. As a whole, $F(p,q)$ is a concave function in $q_i(X_i)$ and therefore has a unique global optimum. As can be seen, $F(p,q)$ gets the optimal value relative only to a single coordinate $q_i(X_i)$ given the choice of all other marginals. Thus to optimize $F(p,q)$ in its entirety, it can be optimized relative to all of the coordinates. Consequently, (\ref{eq:qi_xi}) can be the update criteria.

% Furthermore, the criteria iterations are guaranteed to converge. Because 
Each iteration of the criteria is monotonically nondecreasing in $F(p,q)$ and $F(p,q)$ is bounded ($\leq \ln Z$). Hence, the sequence of distributions must converge with any initial values. The $F(p,q)$'s convergence also implies that the KL-divergence reaches the minimum. 
% So theorem \ref{theo:mean_field} has been proved.

\subsection{Proof of Theorem 2}\label{subsec:initial_solution_proof}
\emph{Proof}: 
A Lyapunov function candidate is chosen as
\begin{equation}\label{eq:Lyapunov function}
    V = \frac{1}{2}\sum _{i = 1}^{N} \left(\sum _{j \in \mathcal{N}_i} \psi _{ar} + 2\psi _{rp} + \mathbf{v}_{ir}^T\mathbf{v}_{ir}\right)
\end{equation}
Let $\tilde{\mathbf{p}}_{i} = \mathbf{p}_{ir}$, $\tilde{\mathbf{v}}_{i} = \mathbf{v}_{ir}$, so the control input $\mathbf{u}_{g,i}^0$ is
\begin{equation}
    \mathbf{u}_{g,i}^0 = - \sum _{j \in \mathcal{N}_i} \left(\nabla _{\tilde{\mathbf{p}}_{i}} \psi _{ar} + k_{av}\tilde{\mathbf{v}}_{ij}\right) - \nabla _{\tilde{\mathbf{p}}_i} \psi _{rp} - k_{rv}^{'}\tilde{\mathbf{v}}_i
\end{equation}
Differentiating  (\ref{eq:Lyapunov function}) with respect to time yields
\begin{equation}
    \begin{aligned}
    \dot{V} &= -\sum _{i = 1}^{N}\sum _{j \in \mathcal{N}_i} k_{av} \tilde{\mathbf{v}}_i^T\tilde{\mathbf{v}}_{ij} - \sum _{i = 1} ^{N} k_{rv}^{'} \tilde{\mathbf{v}}_i^T \tilde{\mathbf{v}}_i \\
    &= -\tilde{\mathbf{v}}^T((k_{av}\mathbf{L} + k_{rv}^{'}\mathbf{I}_3) \otimes \mathbf{I}_3)\tilde{\mathbf{v}} \leq 0
    \end{aligned}
\end{equation}
where $\mathbf{L}$ is the Laplacian matrix of $\mathcal{G}$. Due to the fact that $\mathbf{L}$ is positive semi-definite and symmetric \cite{olfati2006flocking}. Thus $k_{av}\mathbf{L} + k_{rv}^{'}\mathbf{I}_3$ is positive definite and symmetric, $\dot{V} \leq 0$ is always true. So the non-negative function satisfies $V(t) \leq V(0)$ for $t > 0$. Note that $\psi _{ar}, \psi _{rp} \geq 0$, thus inequality $\tilde{\mathbf{v}}_i^T\tilde{\mathbf{v}}_i \leq V(t) \leq V(0)$ holds, which implies $\Vert \tilde{\mathbf{v}}_i \Vert \leq \sqrt{2V(0)}$. $V(0)$ is supposed to be finite, so $\Vert \tilde{\mathbf{v}}_i \Vert \leq \sqrt{2V(0)}$ is bounded. By assumption, $\mathcal{G}$ is connected, hence, any two robots are connected by a path, which guarantees the distance between the two robots is limited. There exists $\Vert \tilde{\mathbf{p}}_{ij} \Vert \leq (N - 1)r_s,\ \forall i,j \in \{1,\ldots, N\}, i \neq j$.
Therefore, the invariant set $\Omega = \{(\tilde{\mathbf{p}}_{ij},\tilde{\mathbf{v}}_i) \vert V(t) < V(0),t \geq 0 \}$ is compact. According to Lasalle's principle, any state starting from set $\Omega$ will asymptotically converge to the largest invariant set in $R = \{(\tilde{\mathbf{p}}_{ij},\tilde{\mathbf{v}}_i) \vert \dot{V}(t) = 0,t \geq 0\}$. Note that $\dot{V} = 0$ is only true at $\tilde{\mathbf{v}} = \mathbf{0}$, \emph{i.e.}, $\mathbf{v}_1 = \ldots = \mathbf{v}_{N} = \mathbf{v}_r$ holds in the steady state. Hence, the velocity of each robot asymptotically converges to the same value $\mathbf{v}_r$. 

\subsection{Proof of Proposition 1}\label{subsec:graph_connected_proof}
\emph{Proof}: This proposition can be proved by contradiction. Assuming $\mathcal{G}$ is not connected, \emph{i.e.}, robot $i$ has no interactions with other robots and it can only be driven by $\mathbf{u}_{grp,i}$ and $\mathbf{u}_{grv,i}$. In fact, the two terms constitute a PD controller, and robot $i$ coincides with the leader when it reaches a steady state, \emph{i.e.}, each robot will approach the leader without the impact of neighboring robots. But there will always be a situation of $\Vert \mathbf{p}_{ij} \Vert \leq r_s$ that is contradictory with the hypothesis. Hence, $\mathcal{G}$ will always tend to be connected. 

\bibliographystyle{IEEEtran}
\bibliography{IEEEabrv,main}

\begin{thebibliography}{10}
\providecommand{\url}[1]{#1}
\csname url@samestyle\endcsname
\providecommand{\newblock}{\relax}
\providecommand{\bibinfo}[2]{#2}
\providecommand{\BIBentrySTDinterwordspacing}{\spaceskip=0pt\relax}
\providecommand{\BIBentryALTinterwordstretchfactor}{4}
\providecommand{\BIBentryALTinterwordspacing}{\spaceskip=\fontdimen2\font plus
\BIBentryALTinterwordstretchfactor\fontdimen3\font minus
  \fontdimen4\font\relax}
\providecommand{\BIBforeignlanguage}[2]{{%
\expandafter\ifx\csname l@#1\endcsname\relax
\typeout{** WARNING: IEEEtran.bst: No hyphenation pattern has been}%
\typeout{** loaded for the language `#1'. Using the pattern for}%
\typeout{** the default language instead.}%
\else
\language=\csname l@#1\endcsname
\fi
#2}}
\providecommand{\BIBdecl}{\relax}
\BIBdecl
\renewcommand{\BIBentryALTinterwordstretchfactor}{4}

\bibitem{bonabeau1999swarm}
E.~Bonabeau, M.~Dorigo, G.~Theraulaz, and G.~Theraulaz, \emph{Swarm
  intelligence: from natural to artificial systems}.\hskip 1em plus 0.5em minus
  0.4em\relax Oxford university press, 1999, no.~1.

\bibitem{8686188}
H.~Xie, X.~Fan, M.~Sun, Z.~Lin, Q.~He, and L.~Sun, ``Programmable generation
  and motion control of a snakelike magnetic microrobot swarm,''
  \emph{IEEE/ASME Transactions on Mechatronics}, vol.~24, no.~3, pp. 902--912,
  2019.

\bibitem{zhang2011spill}
G.~Zhang, G.~K. Fricke, and D.~P. Garg, ``Spill detection and perimeter
  surveillance via distributed swarming agents,'' \emph{IEEE/ASME Transactions
  on Mechatronics}, vol.~18, no.~1, pp. 121--129, 2011.

\bibitem{zhang2017aerodynamics}
Q.~Zhang and H.~H. Liu, ``Aerodynamics modeling and analysis of close formation
  flight,'' \emph{Journal of Aircraft}, vol.~54, no.~6, pp. 2192--2204, 2017.

\bibitem{6732930}
H.~Yu, K.~Meier, M.~Argyle, and R.~W. Beard, ``Cooperative path planning for
  target tracking in urban environments using unmanned air and ground
  vehicles,'' \emph{IEEE/ASME Transactions on Mechatronics}, vol.~20, no.~2,
  pp. 541--552, 2015.

\bibitem{Hu2021TRO}
J.~Hu, P.~Bhowmick, I.~Jang, F.~Arvin, and A.~Lanzon, ``A decentralized cluster
  formation containment framework for multirobot systems,'' \emph{IEEE
  Transactions on Robotics}, vol.~37, no.~6, pp. 1936--1955, 2021.

\bibitem{zhang2021robust}
Q.~Zhang and H.~H. Liu, ``Robust nonlinear close formation control of multiple
  fixed-wing aircraft,'' \emph{Journal of Guidance, Control, and Dynamics},
  vol.~44, no.~3, pp. 572--586, 2021.

\bibitem{9914633}
S.~T. Hart and C.~A. Kitts, ``Unifying control architecture for reactive
  particle swarms,'' \emph{IEEE/ASME Transactions on Mechatronics}, vol.~28,
  no.~2, pp. 873--883, 2023.

\bibitem{reynolds1987flocks}
C.~W. Reynolds, ``Flocks, herds and schools: A distributed behavioral model,''
  in \emph{Proceedings of the 14th annual conference on Computer graphics and
  interactive techniques}, 1987, pp. 25--34.

\bibitem{6876179}
Y.~Jia and L.~Wang, ``Leader–follower flocking of multiple robotic fish,''
  \emph{IEEE/ASME Transactions on Mechatronics}, vol.~20, no.~3, pp.
  1372--1383, 2015.

\bibitem{olfati2006flocking}
R.~Olfati-Saber, ``Flocking for multi-agent dynamic systems: Algorithms and
  theory,'' \emph{IEEE Transactions on automatic control}, vol.~51, no.~3, pp.
  401--420, 2006.

\bibitem{vasarhelyi2018optimized}
G.~V{\'a}s{\'a}rhelyi, C.~Vir{\'a}gh, G.~Somorjai, T.~Nepusz, A.~E. Eiben, and
  T.~Vicsek, ``Optimized flocking of autonomous drones in confined
  environments,'' \emph{Science Robotics}, vol.~3, no.~20, p. eaat3536, 2018.

\bibitem{zhan2013flocking}
J.~Zhan and X.~Li, ``Flocking of multi-agent systems via model predictive
  control based on position-only measurements,'' \emph{IEEE Transactions on
  Industrial Informatics}, vol.~9, no.~1, pp. 377--385, 2013.

\bibitem{7490388}
A.-R. Merheb, V.~Gazi, and N.~Sezer-Uzol, ``Implementation studies of robot
  swarm navigation using potential functions and panel methods,''
  \emph{IEEE/ASME Transactions on Mechatronics}, vol.~21, no.~5, pp.
  2556--2567, 2016.

\bibitem{hu2020distributed}
J.~Hu, J.~Sun, Z.~Zou, D.~Ji, and Z.~Xiong, ``Distributed multi-robot formation
  control under dynamic obstacle interference,'' in \emph{2020 IEEE/ASME
  International Conference on Advanced Intelligent Mechatronics (AIM)}, 2020,
  pp. 1435--1440.

\bibitem{6293885}
J.~Zhan and X.~Li, ``Flocking of multi-agent systems via model predictive
  control based on position-only measurements,'' \emph{IEEE Transactions on
  Industrial Informatics}, vol.~9, no.~1, pp. 377--385, 2013.

\bibitem{6853439}
R.~Negenborn and J.~Maestre, ``Distributed model predictive control: An
  overview and roadmap of future research opportunities,'' \emph{IEEE Control
  Systems Magazine}, vol.~34, no.~4, pp. 87--97, 2014.

\bibitem{yuan2017outdoor}
Q.~Yuan, J.~Zhan, and X.~Li, ``Outdoor flocking of quadcopter drones with
  decentralized model predictive control,'' \emph{ISA transactions}, vol.~71,
  pp. 84--92, 2017.

\bibitem{bennet2010distributed}
D.~J. Bennet and C.~R. McInnes, ``Distributed control of multi-robot systems
  using bifurcating potential fields,'' \emph{Robotics and Autonomous Systems},
  vol.~58, no.~3, pp. 256--264, 2010.

\bibitem{soria2021Nature}
E.~Soria, F.~Schiano, and D.~Floreano, ``Predictive control of aerial swarms in
  cluttered environments,'' \emph{Nature Machine Intelligence}, vol.~3, no.~6,
  pp. 545--554, 2021.

\bibitem{zhang2015model}
H.-T. Zhang, Z.~Cheng, G.~Chen, and C.~Li, ``Model predictive flocking control
  for second-order multi-agent systems with input constraints,'' \emph{IEEE
  Transactions on Circuits and Systems I: Regular Papers}, vol.~62, no.~6, pp.
  1599--1606, 2015.

\bibitem{7574368}
C.~Shen, Y.~Shi, and B.~Buckham, ``Integrated path planning and tracking
  control of an auv: A unified receding horizon optimization approach,''
  \emph{IEEE/ASME Transactions on Mechatronics}, vol.~22, no.~3, pp.
  1163--1173, 2017.

\bibitem{9662427}
M.~Shahriari and M.~Biglarbegian, ``A novel predictive safety criteria for
  robust collision avoidance of autonomous robots,'' \emph{IEEE/ASME
  Transactions on Mechatronics}, vol.~27, no.~5, pp. 3773--3783, 2022.

\bibitem{8429104}
Y.~Liu \emph{et~al.}, ``A distributed control approach to formation balancing
  and maneuvering of multiple multirotor uavs,'' \emph{IEEE Transactions on
  Robotics}, vol.~34, no.~4, pp. 870--882, 2018.

\bibitem{8877998}
Y.~Lyu, J.~Hu, B.~M. Chen, C.~Zhao, and Q.~Pan, ``Multivehicle flocking with
  collision avoidance via distributed model predictive control,'' \emph{IEEE
  Transactions on Cybernetics}, vol.~51, no.~5, pp. 2651--2662, 2021.

\bibitem{shi2021advanced}
Y.~Shi and K.~Zhang, ``Advanced model predictive control framework for
  autonomous intelligent mechatronic systems: A tutorial overview and
  perspectives,'' \emph{Annual Reviews in Control}, vol.~52, pp. 170--196,
  2021.

\bibitem{koller2009probabilistic}
D.~Koller and N.~Friedman, \emph{Probabilistic graphical models: principles and
  techniques}.\hskip 1em plus 0.5em minus 0.4em\relax MIT press, 2009.

\bibitem{xi2006gibbs}
W.~Xi, X.~Tan, and J.~S. Baras, ``Gibbs sampler-based coordination of
  autonomous swarms,'' \emph{Automatica}, vol.~42, no.~7, pp. 1107--1119, 2006.

\bibitem{tan2010decentralized}
X.~Tan, W.~Xi, and J.~S. Baras, ``Decentralized coordination of autonomous
  swarms using parallel gibbs sampling,'' \emph{Automatica}, vol.~46, no.~12,
  pp. 2068--2076, 2010.

\bibitem{rezeck2021cooperative}
P.~Rezeck, R.~M. Assun{\c{c}}{\~a}o, and L.~Chaimowicz, ``Cooperative object
  transportation using gibbs random fields,'' in \emph{2021 IEEE/RSJ
  International Conference on Intelligent Robots and Systems (IROS)}, 2021, pp.
  9131--9138.

\bibitem{fernando2021online}
M.~Fernando, ``Online flocking control of uavs with mean-field approximation,''
  in \emph{2021 IEEE International Conference on Robotics and Automation
  (ICRA)}, 2021, pp. 8977--8983.

\bibitem{580977}
D.~Fox, W.~Burgard, and S.~Thrun, ``The dynamic window approach to collision
  avoidance,'' \emph{IEEE Robotics \& Automation Magazine}, vol.~4, no.~1, pp.
  23--33, 1997.

\bibitem{cole2018reactive}
K.~Cole and A.~M. Wickenheiser, ``Reactive trajectory generation for multiple
  vehicles in unknown environments with wind disturbances,'' \emph{IEEE
  Transactions on Robotics}, vol.~34, no.~5, pp. 1333--1348, 2018.

\bibitem{rezeck2021flocking}
P.~Rezeck, R.~M. Assun{\c{c}}{\~a}o, and L.~Chaimowicz, ``Flocking-segregative
  swarming behaviors using gibbs random fields,'' in \emph{2021 IEEE
  International Conference on Robotics and Automation (ICRA)}, 2021, pp.
  8757--8763.

\bibitem{7434002}
H.~Fang, Y.~Wei, J.~Chen, and B.~Xin, ``Flocking of second-order multiagent
  systems with connectivity preservation based on algebraic connectivity
  estimation,'' \emph{IEEE Transactions on Cybernetics}, vol.~47, no.~4, pp.
  1067--1077, 2017.

\bibitem{10026865}
J.~Guo, J.~Qi, M.~Wang, C.~Wu, and G.~Yang, ``Collision-free distributed
  control for multiple quadrotors in cluttered environments with static and
  dynamic obstacles,'' \emph{IEEE Robotics and Automation Letters}, vol.~8,
  no.~3, pp. 1501--1508, 2023.

\bibitem{fernando2020swarming}
M.~Fernando and L.~Liu, ``Swarming of aerial robots with markov random field
  optimization,'' \emph{CoRR}, vol. abs/2010.06274, 2020, (arXiv: 2010.06274).

\end{thebibliography}

\begin{IEEEbiography}[{\includegraphics[width=1in,height=1.25in,clip,keepaspectratio]{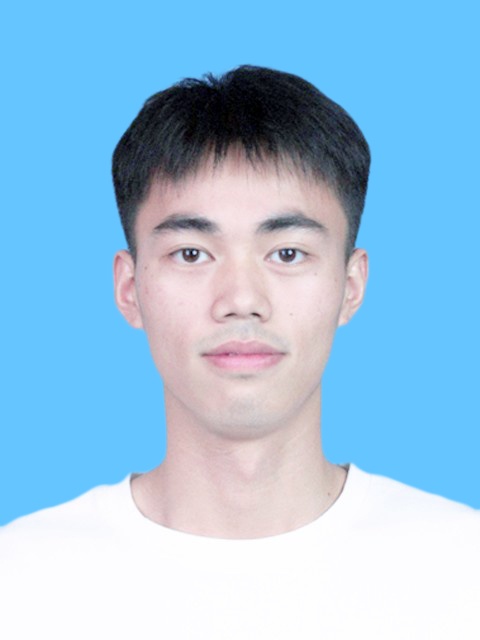}}]{Guobin Zhu} received his B.S. degree in the School of Aeronautics and Astronautics, Sun Yat-sen University, Shenzhen, China, in 2023. He is currently a doctoral candidate at the School of Automation and Electrical Engineering, Beihang University, Beijing, China. His research interests include reinforcement learning, multi-robot/agent systems, and unmanned aerial vehicle systems.
\end{IEEEbiography}

\begin{IEEEbiography}[{\includegraphics[width=1in,height=1.25in,clip,keepaspectratio]{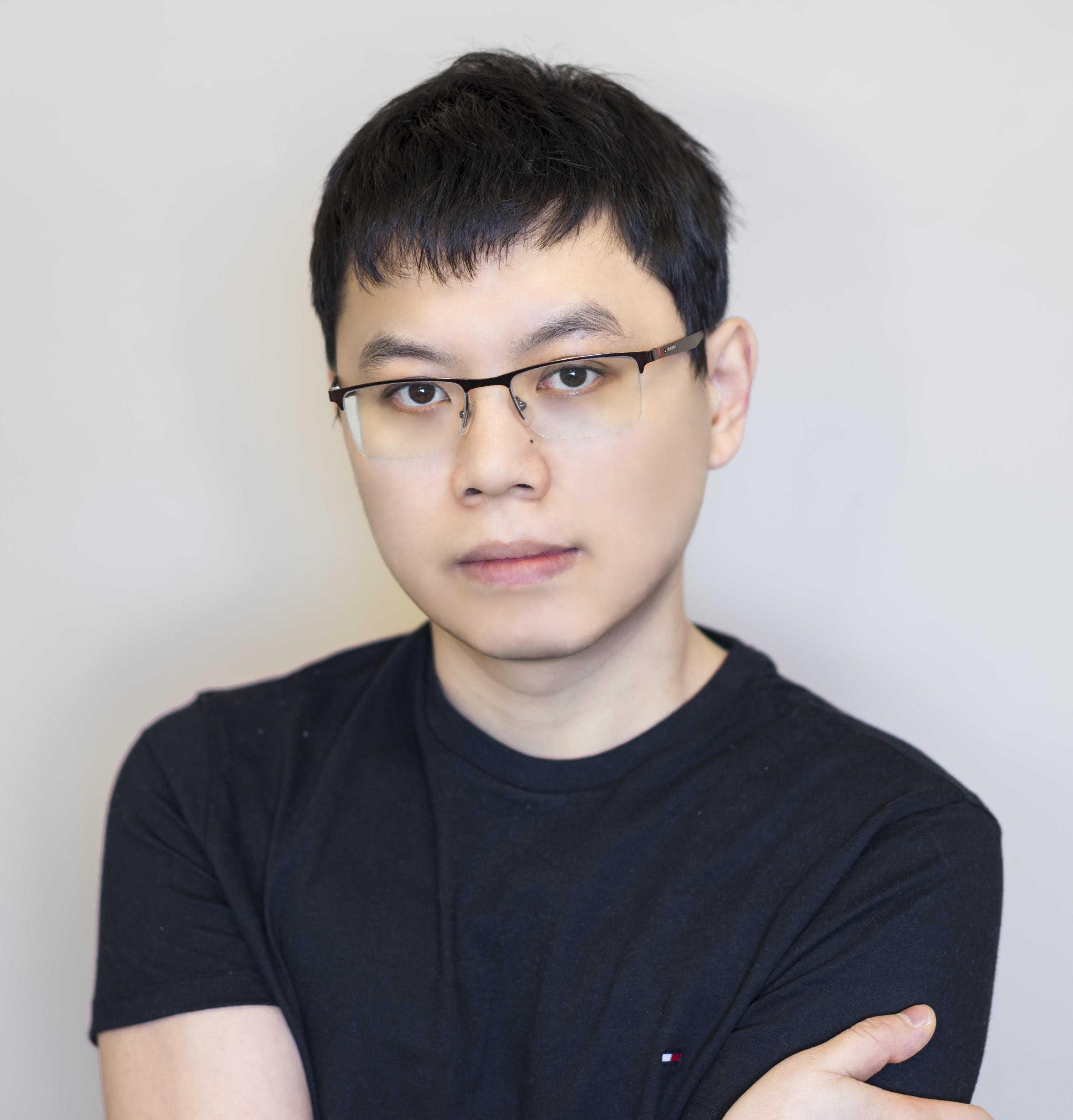}}]{Qingrui Zhang} (S'17-M'20) received his B.S. degree in automatic control from Harbin Institute of Technology, Harbin, China, in 2013, and Ph.D. degree in Aerospace Science and Engineering from University of Toronto, Toronto, ON, Canada, in 2019. From 2019 to 2020, he was a postdoctoral research fellow at Delft University of Technology (TU Delft), Delft, the Netherlands. Since 2020, he has been with the School of Aeronautics and Astronautics, Sun Yat-sen University, Shenzhen, China, where he is currently an Associate Professor. He was the recipient of the Gordon N. Patterson Student Award for the top Ph.D. graduate from the University of Toronto Institute for Aerospace Studies (UTIAS), in 2019.  His research interests include reinforcement learning, learning/model-based optimal control, multi-robot/agent systems, and unmanned aerial vehicles.
\end{IEEEbiography}

\begin{IEEEbiography}[{\includegraphics[width=1in,height=1.25in,clip,keepaspectratio]{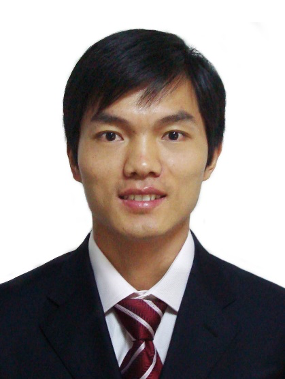}}]{Bo Zhu} received the B.E. and Ph.D. degrees both from Beihang University, Beijing, China, in 2004 and 2010, respectively. He was a Lecturer and then promoted to an Associate Professor in July 2013, with the School of Aeronautics and Astronautics, University of Electronic Science and Technology of China (UESTC), Chengdu, China. He visited the FSC Lab, University of Toronto Institute for Aerospace Studies (UTIAS), Toronto, ON, Canada, from 2013 to 2014. Since March 2019, he has been an Associate Professor and Ph.D. supervisor with the School of Aeronautics and Astronautics, Sun Yat-sen University (SYSU), Guangzhou, China. His research interests include uncertainty and disturbance estimation (UDE) theory and technique; and autonomous, dependable, and affordable swarm systems (ADA-SS).
\end{IEEEbiography}

\begin{IEEEbiography}
[{\includegraphics[width=1in,height=1.25in,clip,keepaspectratio]{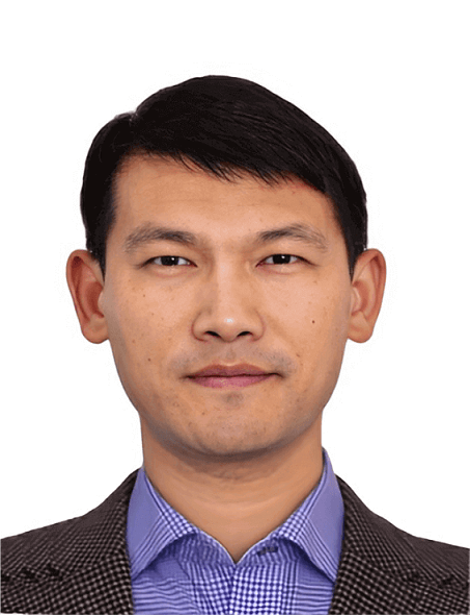}}]{Tianjiang Hu} (Member, IEEE) received the B.Eng. and Ph.D. degrees in robotics and automatic control from the National University of Defense Technology, Changsha, China, in 2002 and 2009, respectively. He was a joint Ph.D. candidate with the Nanyang Technological University, Singapore (2007–2008), financially supported by the Chinese Scholarship Council.

Dr. Hu is currently a Full Professor with Sun Yat-sen University, China where he has served as the founder of Machine Intelligence and Collective Robotics (MICRO) Lab. He has also been a Visiting Scholar for international collaboration with Nanyang Technological University, Singapore, and the University of Manchester, U.K. He has published over 20 technical papers in refereed international journals and academic conference proceedings. His current research interests include collective behavior, bio-inspired robotics, autonomous systems, and learning control.

Dr. Hu received the Best Paper Award at the International Conference
on Realtime Computing and Robotics (RCAR 2015). He has served as the
General Chair of the inaugural 2021 International Conference on Swarm
Intelligence and Collective Robotics (SICRO).
\end{IEEEbiography}
\end{document}